\documentclass[twocolumn,10pt]{article}
\usepackage[margin=0.5in]{geometry}
\usepackage{authblk}

\usepackage[utf8]{inputenc} % allow utf-8 input
\usepackage{float}
\usepackage[T1]{fontenc}    % use 8-bit T1 fonts
\usepackage{hyperref}       % hyperlinks
\usepackage{url}            % simple URL typesetting
\usepackage{booktabs}       % professional-quality tables
\usepackage{nicefrac}       % compact symbols for 1/2, etc.
\usepackage{microtype}      % microtypography
\usepackage{amsmath, amsfonts, graphicx}

\usepackage{graphicx}

\usepackage{amsbsy}
\usepackage{amssymb}
\usepackage{amsmath}
\usepackage{amsthm}
\usepackage{color}
\usepackage{tabularx}
\usepackage{booktabs}
\usepackage[dvipsnames]{xcolor}

\usepackage{algorithm,algorithmic}

\usepackage{tikz,pgf,tikz-3dplot}
\usepackage{tikz}

\def\td{{\text d}}

\def\dt{{\text dt}}

\def\L{{\mathbb L}}

\def\R{{\mathbb R}}

\def\PP{ \mathcal{P} }

\def\NN{ \mathcal{N} }

\def\LL{ \mathcal{L} }

\newcommand{\WF}[2]{\operatorname{WF}\left(\left[#1\right],\left[#2\right]\right)}
\newcommand{\KL}{\operatorname{KL}}
\newcommand{\IS}{\operatorname{IS}}

\newcommand{\kNN}{\operatorname{KNN}}

\def\byn{ \boldsymbol{y}_{\boldsymbol{n}} }
\def\brn{ \boldsymbol{r}_{\boldsymbol{n}} }
\def\bcn{ \boldsymbol{c}_{\boldsymbol{n}} }
\def\bsn{ \boldsymbol{s}_{\boldsymbol{n}} }

\DeclareMathOperator*{\argmin}{arg\,min}

\newcommand{\E}[1]{\mathbb E \left[#1\right]}

\newtheorem{proposition}{Proposition}  
\newtheorem{definition}{Definition}
\newtheorem{theorem}{Theorem}
\newtheorem{remark}{Remark}

\newcommand{\CR}[1]{{\color{black} #1}}

\newcommand{\bb}[1]{{\color{blue} \textbf{#1}}}
\newcommand{\br}[1]{{\color{red} \textbf{#1}}}

\begin{document}

\title{The Wasserstein-Fourier Distance \\for Stationary Time Series\footnote{This work was produced when E.~Cazelles and A.~Robert were with the Center for Mathematical Modeling. This work was supported by Fondecyt-Postdoctorado \#3190926, Fondecyt-Iniciación \#11171165, ANID-AFB170001 Center for Mathematical Modeling, and ANID-FB0008 Advanced Center for Electrical and Electronic Engineering. Corresponding author: Felipe Tobar (ftobar@dim.uchile.cl).}}

\author[1]{Elsa Cazelles}
\author[2]{Arnaud Robert}
\author[3]{Felipe Tobar}
\affil[1]{IRIT,  Universit\'{e}  de  Toulouse, CNRS}
\affil[2]{Brain $\&$ Behaviour Lab, Dept. of Computing, Imperial College London}
\affil[3]{Center for Mathematical Modeling, Universidad de Chile}

%\author{Elsa Cazelles, Arnaud Robert, Felipe Tobar
%RIT,  Universit\'{e}  de  Toulouse,  CNRS, Brain $\&$ Behaviour Lab, Dept. of Computing, Imperial College London, Center for Mathematical Modeling and the Department of Mathematical Engineering, Universidad de Chile.}
% <-this % stops a space

\maketitle

%!TEX root = ../CRT_TSP2019_arxiv.tex

\begin{abstract}
We propose the Wasserstein-Fourier (WF) distance to measure the (dis)similarity between time series by quantifying the displacement of their energy across frequencies. The WF distance operates by calculating the Wasserstein distance between the (normalised) power spectral densities (NPSD) of time series. Yet this rationale has been considered in the past, we fill a gap in the open literature providing a formal introduction of this distance, together with its main properties from the joint perspective of Fourier analysis and optimal transport. As the main aim of this work is to validate WF as a general-purpose metric for time series, we illustrate its applicability on three broad contexts. First, we rely on WF to implement a PCA-like dimensionality reduction for NPSDs which allows for meaningful visualisation and pattern recognition applications. Second, we show that the geometry induced by WF on the space of NPSDs admits a \emph{geodesic} interpolant between time series, thus enabling data augmentation on the spectral domain, \CR{by averaging the dynamic content of two signals.} Third, we implement WF for time series classification using parametric/non-parametric classifiers and compare it to other classical metrics. Supported on theoretical results, as well as synthetic illustrations and experiments on real-world data, this work establishes WF as a meaningful and capable resource pertinent to general distance-based applications of time series. 

\end{abstract}

%\IEEEpeerreviewmaketitle

%!TEX root = ../CRT_TSP2019_arxiv.tex

\section{Introduction} 
\label{sec:intro}  

Time series analysis is ubiquitous in a number of applications: from seismology to audio enhancement, from astronomy to fault detection, and from underwater navigation to source separation. Time series analysis occurs naturally in the spectral domain, where signals are represented  in terms of how their energy is spread across frequencies. The spectral content of a signal is given by its---e.g., Fourier---power spectral density (PSD), which can be computed via the Periodogram \cite{periodogram} or parametric approaches such as the autoregressive and Yule-Walker methods \cite{Yule267,Walker518}. Comparing time series through their respective PSDs has been considered since the late 1960's, e.g., using the Itakura-Saito (IS) distance \cite{itakura1968analysis} given by the Bregman divergence associated to the logarithm function between PSDs. Today, comparing time series in the spectral domain is still relevant mainly due to the recent advances on spectral estimation that are robust to noise, missing values, and general acquisition \emph{artefacts} \cite{turner_sahani,choudhuri_2004,protopapas,BABU2010359,tobar18}. However, the IS divergence (as well as most distances between functions) is of no use when the spectral supports\footnote{The \emph{support} of a function is the subset of the input space where the function is strictly greater than zero, therefore, the spectral support is the set of frequency components that have energy and thus are present in the time series.} of the functions differ. Therefore, developing sound metrics for comparing general PSDs remains essential. 

Beyond the realm of spectral analysis, the \emph{de facto} tool for comparing general densities are Bregman divergences \cite{amari2007methods, banerjee2005clustering}. The most popular one being the Kullback-Leibler (KL) divergence, which requires that one distribution \emph{dominates}\footnote{Distribution $p$ dominates distribution $q$ if the support of $q$ is included in the support of $p$.} the other. To a lesser extent, the Euclidean distance has also been used for general-purpose comparison of distributions. In this article, we claim that the KL, IS and Euclidean distances have a fundamental drawback when comparing distributions of power, and thus they impose stringent requirements on the densities under study. In particular, these distances are only meaningful when the densities share a common support. In a nutshell, our point of view follows from the fact that the KL, IS and Euclidean distances are \emph{vertical} divergences, meaning that they compare the (point-wise) values of two distributions for a given location in the $x$-axis (frequency). We, however, focus on how spectral energy is distributed across different values of the $x$-axis, thus requiring a \emph{horizontal} distance. Fig.~\ref{fig:hor_vert} illustrates the differences between a vertical and a horizontal distance using two Gaussians: observe that when these Gaussians become farther apart, the vertical distance becomes insensitive, whereas the horizontal one faithfully represents the \emph{displacement} regardless of how far apart the densities are.
\begin{figure}[!ht]
\centering
\includegraphics[height=4em]{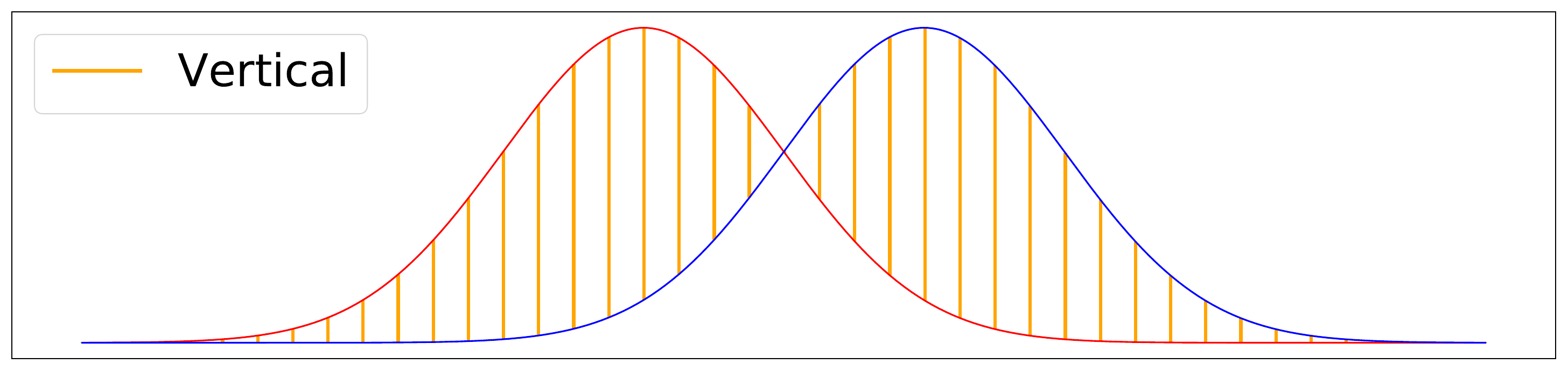}
\includegraphics[height=4em]{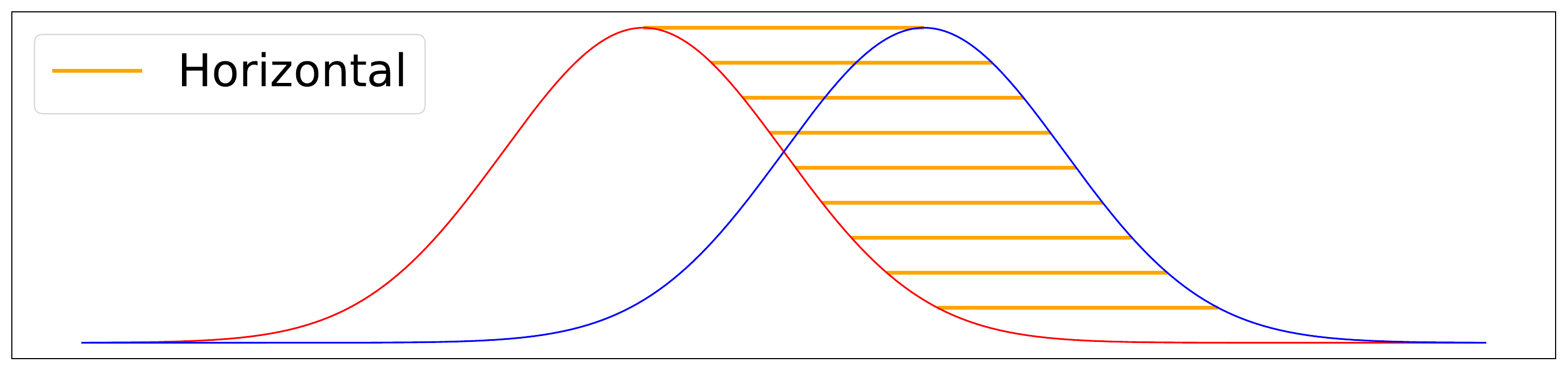}
  \caption{Vertical (top) and horizontal (bottom) displacements of two Gaussian distributions.}
  \label{fig:hor_vert}
\end{figure}

With the objective of building a meaningful, horizontal, spectral distance between time series, we rely on optimal transport (OT) and in particular on the Wasserstein distance \cite{villani2008optimal}, which is given by the cost of the \emph{optimal transportation} (OT) between two probability distributions. This distance has already proven successful in a number of unsupervised learning applications such as clustering (of distributions) \cite{ye2017fast}, generative adversarial networks \cite{arjovsky2017wasserstein} and Bayesian model selection \cite{BLWB}. Using probability distances for PSDs is also supported by the work of Basseville in \cite{basseville1989distance}, who in the late 1980's had already categorised the distances for PSDs as \emph{probability based} or \emph{frequency based}. In the same line, the use of the Wasserstein distance on the space of PSDs is not new, related works include those of Flamary et al.~in \cite{flamary2016optimal}, which built a dictionary of fundamental and harmonic frequencies thus emphasising the importance of moving mass along the frequency dimension, and also  \cite{rolet2018blind}, that followed the same rationale for supervised speech blind source separation. However, besides specific-purpose previous works, the open literature is still lacking a systematic study of the properties of this distance and its competence in general-purpose time series applications.

\textbf{Contributions.} Our hypothesis is that a rigorous framework for applying the Wasserstein distance to Fourier (power) representations, as well as understanding its geometric and statistical properties, would lay the foundations for developing novel tools for time series analysis. In that sense, the main contribution of this article is to propose an alternative, largely unexplored so far, general-purpose distance between time series based on OT. Additionally, specific contributions that are, to the best of our knowledge, novel in the literature include: 
\begin{itemize}
	\item The formal definition of the Wasserstein-Fourier (WF) distance and presentation of properties derived from a joint OT and spectral analysis standpoint (Sec.~\ref{sec:properties})
	\item The convergence behaviour of the WF distance and how it relates to the convergence of autocorrelation functions (Sec.~\ref{sec:convergence}, Prop.~\ref{prop:conv})
	\item A \emph{geodesic}, i.e., minimum-cost, interpolant between time series based on the WF distance. \CR{This interpolation reveals the dynamic evolution between signals,} and can be used for data augmentation (Sec.~\ref{sec:geodesic})
	\item A PCA-like dimensionality reduction technique based on the WF distance, specially suited for spectral representations, that allows for visualisation and pattern recognition (Sec.~\ref{sec:PCA})
	\item A set of classifiers for time series by equipping standard parametric/non-parametric classifiers with the WF distance (Sec.~\ref{sec:log_reg})
	\item All the above conceptual packages are equipped with practical validation on illustrative synthetic examples and real-world data throughout the paper. \CR{The code to reproduce the experiments is available in Python at \url{https://github.com/GAMES-UChile/Wasserstein-Fourier}.}
\end{itemize}

\textbf{Organisation.} The rest of this article is organised as follows. Section \ref{sec:background} presents the technical background for Fourier analysis and OT required to present our contribution. Then, Section \ref{sec:definition} presents the proposed Wasserstein-Fourier distance and analyses its properties from OT and Signal Processing perspectives. Sections \ref{sec:geodesic}, \ref{sec:PCA} and \ref{sec:log_reg} are devoted to the validation of the proposed distance through the three distance-based sets of experiments using synthetic and real-world datasets: interpolation paths between time series (and Gaussian processes), spectral dimensionality reduction and parametric/non-parametric regression. The paper then concludes with an assessment of our findings and a proposal for future research activities.

%!TEX root = ../CRT_TSP2019_arxiv.tex

\section{Background, assumptions and desiderata}
\label{sec:background} 

%Recall that our focus is to compare time series by computing the Wasserstein distance between their respective PSDs. However, as the Wasserstein distance is defined for probability distributions (rather than PSDs), we revise the conditions to appropriately equip PSDs with the Wasserstein metric.

\subsection{An equivalence class for time series}
\label{sec:class}

We consider \textit{Fourier} PSDs \cite{kay:88}, based on the Fourier transform given by $\hat{x}(\xi)=\int_{-\infty}^{\infty}x(t)e^{-j2\pi\xi t}\dt$ (we denote by $j$ the imaginary unit). The Fourier PSD for the continuous-time, stationary, stochastic  time series $x(t)$ considered in this work is then given by
\begin{equation}
  S(\xi) = \lim_{T\rightarrow\infty}  \E{\frac{1}{2T}\left|\int_{-T}^T x(t)e^{-j2\pi t \xi}\td t\right|^2}.
\end{equation}
\begin{remark}

Though our work is presented for in the continuous-time, univariate, Fourier setting, the proposed WF distance can be applied to multivariate input, discrete time and arbitrary spectral domains. \CR{Critically, the WF distance is able to compare discrete-time signals against continuous-time ones. The only requirement of our method is that the spectral energy exists and it is finite, for instance, in the discrete-time case this translates to absolute summability and stationarity.} Furthermore, the scope of our work is on  wide-sense stationary signals, henceforth referred to as \emph{stationary}. 

\end{remark}
We consider a unit-norm version of the Fourier PSD termed normalised power spectral density (NPSD) given by
\begin{equation}
   s(\xi) = \frac{S(\xi)}{\int S(\xi')\td \xi'}.
   \label{eq:def_npsd}
\end{equation}
The justification for this choice is twofold. First, representing time series by a NPSD allows us to exploit the vast literature on divergences for probability distributions. Second, to compare the dynamic content of time series, i.e., how they evolve in time, the magnitude of the time series is irrelevant and so is the magnitude of its PSD (due to the linearity of the Fourier transform).

Choosing NPSDs and neglecting both the signal's power and phase makes our representation blind to time-shifts (phase) and scaling. For stochastic processes, this representation has an interesting interpretation: as the PSD of a stationary stochastic process is uniquely linked to its covariance kernel (via Bochner's theorem \cite{salomon}), all processes with proportional kernels will have the same NPSD representation. This is particularly useful for Gaussian process (GPs, \cite{Rasmussen:2006}), as it allows us to define a proper distance between GPs---see Sec.~\ref{ssub:gps}. 

The NPSD representation thus induces an equivalence class for time series, denoted by $[x] = \{ x' | s_{x'}=s_x\}$, where $s_x$ denotes the NPSD of $x$. The NPSD is then a non-injective embedding (or \emph{projection}) from the space of time series onto the space of probability distributions due to the unit-norm choice. Time series belonging to the same equivalence class are thus understood to have the same dynamic content. As a consequence, we can define a distance between two (equivalence) classes of time series as the distance between their corresponding NPSDs, where the latter can be chosen as a divergence for probability distributions such as the Euclidean, Kullback-Leibler or Itakura-Saito divergences. In this work, we consider the Wasserstein distance instead.

\subsection{Optimal transport}

Optimal transport (OT, \cite{villani2008optimal}) compares two probability distributions paying special attention to the geometry of their domain. This comparison is achieved by finding the lowest cost to transfer the mass from one probability distribution onto the other. Specifically, the \emph{optimal transport} between two measures $\mu$ and $\nu$ defined on $\R^d$ is given by
\begin{equation}
\label{def:OTgeneral}
%W_p(\mu,\nu) = \left(\underset{\pi\in\Pi(\mu,\nu)}{\min} \int_{\R^d\times\R^d}c(x,y)\td\pi(x,y)\right)^{1/p},
\underset{\pi\in\Pi(\mu,\nu)}{\min} \int_{\R^d\times\R^d}c(x,y)\td\pi(x,y),
\end{equation}
where $\pi$ is a joint distribution for $x$ and $y$, referred to as the \emph{transport plan}, belonging to the space $\Pi(\mu,\nu)$ of the product measures on $\R^d\times\R^d$ with marginals $\mu$ and $\nu$; and $c:\R^d\times\R^d\mapsto\R$ is a cost function.
%This formulation is a generalisation of the Monge's problem defined as
%\begin{equation}
%\label{def:OTKant}
%W_p(\mu,\nu) = \left(\underset{T}{\min} \int_{\R^d}c(x,T(x))d\mu(x)\right)^{1/p}
%\end{equation}
%where $T:\R^d\to\R^d$ belongs to the set of measurable functions such that $\nu=T\#\mu$ (meaning that for any measurable set $B \subset\R^d$, we have $\nu(B)=\mu(T^{-1}(B))$); note that this set can be empty.
In this manuscript, we focus on the \emph{Wasserstein distance}, or Earth Mover's distance \cite{rubner2000earth}, which is obtained by setting $c(x,y)=\Vert x-y\Vert^p$, $p\geq 1$, in eq.~\eqref{def:OTgeneral}. This distance, denoted by 
\begin{equation}
\label{def:OT}
W_p(\mu,\nu) = \left(\underset{\pi\in\Pi(\mu,\nu)}{\min} \int_{\R^d\times\R^d}\Vert x-y\Vert^p\td\pi(x,y)\right)^{1/p},
\end{equation}
\emph{metrises} the space $\PP_p(\R^d)$ of measures admitting moments of order $p\geq 1$. Remarkably, OT immediately addresses the key challenge highlighted in Sec.~\ref{sec:intro}, as it allows us to compare meaningfully continuous and discrete densities of different supports unlike the Euclidean distance and KL and IS divergences.%{}

Additionally, notice that for one-dimensional probability measures---such as the NPSDs of a one dimensional signal, the optimisation \eqref{def:OT} is closed-form and boils down to 
\begin{equation}
\label{def:Wdim1}
	W_p(\mu,\nu) = \left(\int_0^1 \vert F^{-}_{\mu}(t)-F^{-}_{\nu}(t)\vert^p\td t\right)^{1/p},	
\end{equation}
where $F^{-}_{\mu}$ represents the inverse cumulative function of $\mu\in\PP_p(\R)$; this means that  the calculation of the 1D Wasserstein distance is computationally inexpensive.  We refer the reader to \cite{villani2008optimal} for a theoretical presentation of OT and to \cite{computationalOT} for recent computational aspects.

%!TEX root = ../CRT_TSP2019_arxiv.tex

\section[A Wasserstein-based distance between Fourier power spectra]{A Wasserstein-based distance between \\ Fourier power spectra}
\label{sec:definition}
Let us consider the stationary signals $x$ and $y$ belonging to two classes of time series denoted by $[x]$ and $[y]$ respectively.
\begin{definition}
The Wasserstein-Fourier (WF) distance between two equivalence classes of time series $[x]$ and $[y]$ is given by
\begin{equation}
\label{def:distance}
\WF{x}{y} = W_2(s_x,s_y),
\end{equation}
where $s_x$ and $s_y$ denote the NPSDs associated to $[x]$ and $[y]$ respectively, as defined in eq.~\eqref{eq:def_npsd}. 
\end{definition}
Note that by a slight abuse of notation, $s_x$ denotes both the probability density (or mass) function and the measure itself.

\begin{remark}
The proposed WF distance holds for multi-dimensional signals, as long as their NPSDs exist. \CR{The contributions of this paper can then be straightforwardly applied to multi-dimensional signals, by considering NPSDs that are distributions supported on $\R^d, d>1$, and for which the Wasserstein distance can be computed.} However, in  most time series applications presented here, such as the experiments in Sec. \ref{sec:geodesic}, \ref{sec:PCA} and \ref{sec:log_reg}, it is only needed to compute PSDs of uni-dimensional data. In such cases, the computational cost of the proposed WF distance is only of linear order.
\end{remark}

Since the function $\operatorname{WF}(\cdot,\cdot)$ inherits the properties of the Wasserstein distance, we directly obtain the following result.
\begin{theorem}
\label{th:WFdistance}
$\operatorname{WF}(\cdot,\cdot)$ is a distance over the space of equivalence classes of time series.
\end{theorem}
\begin{proof}
The proof follows from the properties of the $W_2$ distance. For three arbitrary time series $x\in [x]$, $y\in [y]$ and $z\in [z]$, WF verifies:\\
\textbf{(i)} non-negativity: $\WF{x}{y}\geq 0$ is direct by the non-negativity of $W_2$,\\
\textbf{(ii)} identity of indiscernible: $\WF{x}{y}=W_2(s_x,s_y)=0$ is equivalent to $s_x=s_y$, and by definition of the equivalence class, to $[x]=[y]$,\\
\textbf{(iii)} symmetry: $\WF{x}{y}=W_2(s_x,s_y)=W_2(s_y,s_x)=\WF{y}{x}$,\\
\textbf{(iv)} triangle inequality: $\WF{x}{y}=W_2(s_x,s_y)\leq W_2(s_x,s_z)+W_2(s_z,s_y)=\WF{x}{z}+\WF{z}{y}$,
which concludes the proof.
\end{proof}

\subsection{Properties of the WF distance}
\label{sec:properties}

We next state key features of the proposed WF distance, which follow naturally from the properties of the Wasserstein distance and the Fourier transform and, therefore, are of general interest for the interface between signal processing and OT. In what follows, we consider two arbitrary (possibly complex-valued) time series denoted $x(t)$ and $y(t)$, also, we denote the PSD of $x$ as $S_x$ and its NPSD as $s_x$.

\textbf{\CR{Time shift.}} If $y(t)$ is a time-shifted version of $x(t)=y(t-t_0)$, their Fourier transforms are linked by $\hat{x}(\xi)=e^{2j\pi t_0\xi}\hat{y}(\xi)$ and since $\vert \hat{x}(\xi)\vert = \vert e^{2j\pi t_0\xi}\vert\times\vert\hat{y}(\xi)\vert= \vert\hat{y}(\xi)\vert$, their PSDs are equal $S_x=S_y$. Therefore, 
\begin{equation}
	\WF{x}{y}=0,
\end{equation}
which is in line with our aim (Sec.\ref{sec:class}) of constructing a distance that is invariant under time shifts.

\textbf{Time scaling.} If $x(t) = y(at), a>0$, we obtain $\hat{x}(\xi)=\frac{1}{a}\hat{y}(\frac{\xi}{a})$. \CR{The PSD is therefore given by $S_x(\xi)=\frac{1}{a^2}S_{y}(\frac{\xi}{a})$ and $\int S_x(u)\td u=\frac{1}{a}\int S_y(u)\td u$,} leading to an NPSD $s_x(\xi) = \frac{1}{a}s_y\left(\frac{\xi}{a}\right)$. In this case, the WF distance translates the scaling effect of magnitude $a$, since the cumulative distribution functions write $F_x(\xi)=F_y(\frac{\xi}{a})$, and therefore the inverse cumulative functions obey  $F^{-}_x(t):=\inf\{\xi\in\R\ : \ F_y(\frac{\xi}{a})\geq t \}=a F^{-}_y(t)$. Thus,
\begin{align}
	\WF{x}{y} &= \left(\int_0^1\vert a F^{-}_y(t)-F^{-}_y(t)\vert^2\td t\right)^{1/2}\nonumber\\
	&=\vert a-1\vert \left(\int_0^1\vert F^{-}_y(t)\vert^2\td t\right)^{1/2}\nonumber\\
	&=\vert a-1\vert \left(\int \vert u\vert^2 \td s_y(u)\right)^{1/2},
\end{align}
where we applied the change of variable $u=F^{-}_y(t)$.
%where $Y$ is a random variable with probability distribution $s_y$.

\textbf{\CR{Frequency shift.}} If $x(t)=e^{2j\pi \xi_0 t}y(t)$ is a frequency-modulated version of the signal $y$, their NPSDs then verify $s_x(\xi)= s_y(\xi-\xi_0)$, which corresponds to a translation of the densities. \CR{Therefore, their cumulative distribution functions read $F^{-}_x(t) =  F^{-}_y(t)+\xi_0$, and Definition \eqref{def:Wdim1} leads to}
\begin{equation}
	\WF{x}{y} = \vert \xi_0\vert.
\end{equation}

\textbf{Non-convexity.} Observe that convexity (in each variable) of $W_p$ does not imply convexity of the WF distance as shown in the following counter-example. Let us consider the signals:
\begin{align}
	x(t) &= \cos(\omega t)\label{eq:test_s1}\\
	y^{\pm}(t) &= \pm x(t)+\frac{1}{r}\cos(\alpha t)\label{eq:test_s2},
\end{align} 
with  $r>0$. Then for $r$ large enough we have
\begin{align}
	\WF{\frac{y^{+}+y^{-}}{2}}{x}   &=\vert \omega-\alpha\vert\label{eq:nonconv}\\ 
									&> \vert \omega-\alpha\vert / (2\pi\sqrt{r^2+1})\nonumber \\
							        &= \frac{\WF{y^{+}}{x}+\WF{y^{-}}{x}}{2}.\nonumber
\end{align}
We explain the last equality in detail in Appendix \ref{appendix:nonconv}. In particular, the convexity of $W_p$ is not preserved by WF due to the normalisation of the PSDs.

\subsection{Convergence properties in connection to the time domain} 
\label{sec:convergence}

In order to validate the WF distance as a sound metric for statistical time-series, we now analyse the convergence of the empirical NPSD (to the true NPSD) in Wasserstein distance, in relationship to the pointwise convergence of the empirical autocovariance function of a stationary time series. 

The autocovariance function (at lag $h$) of a stationary zero-mean time series $y\in [y]$ is defined by 
\begin{equation}
	c(h) = \E{y(t)y^{\ast}(t+h) },	
\end{equation}
and its normalised version referred to as \textit{autocorrelation function (ACF)} is given by 
\begin{equation}
 	r(h) = c(h)/c(0).\label{eq:acf}
 \end{equation}
Furthermore, by Bochner theorem \cite{salomon} we have that 
\begin{equation}
\label{eq:Wiener}
S(\xi) = \int_{\R} c(h) e^{-j2\pi h\xi}\td h,
\end{equation}
where $S$ denotes the PSD of $y$. Applying the inverse Fourier transform (given by $x(t) = \int \hat{x}(\xi) e^{2\pi j t\xi}\td\xi$) to the above equation leads to $c(0)=\int S(\xi)\td\xi$, meaning that the normalising constant of the ACF in eq.~\eqref{eq:acf} is the one required to compute the NPSD proposed in eq.~\eqref{eq:def_npsd}. As a consequence, we can interpret a \emph{normalised version} of Bochner theorem that associates the NPSD with the ACF   by
\begin{equation}
\label{eq:Wiener2}
s(\xi) = \int_{\R} r(h) e^{-j2\pi h\xi}\td h,
\end{equation}
rather than the PSD with the covariance as the original theorem.

Leaning on this relationship, the following convergence result ties the sample ACF to the sample time series. For the sake of simplicity, we consider a zero-mean discrete-time series (i.e., a vector of infinitely-countable, absolutely summable,  entries).

\begin{proposition}
\label{prop:conv}
Let $y$ be a discrete-time zero-mean series and $\byn=[y_1,\ldots,y_n]$ be a sample of $y$, with $r$ and $\brn$ their true and empirical\footnote{We denote the empirical ACF $\brn$ and covariance $\bcn$ by (see e.g. \cite{chatfield2016analysis})
$$\brn(h) = \frac{\bcn(h)}{\bcn(0)},\quad \mbox{where}\quad \bcn(h)=\frac{1}{n}\sum_{t=0}^{n- h }y_ty^{\ast}_{t+h}.$$} ACFs respectively. Then, if $y$  is band limited and stationary, we have
%\begin{enumerate}
%	\item[\textbf{(i)}] if $y$ is a band-limited process and $\underset{n\to\infty}{\lim} \brn(h) = r(h)$ then $\underset{n\to\infty}{\lim} \WF{\byn}{y} = 0$
%\item[\textbf{(ii)}] $\underset{n\to\infty}{\lim} \WF{\byn}{y} = 0$ implies that $\underset{n\to\infty}{\lim} \brn(h) = r(h)$.
%\end{enumerate}
\begin{equation}
	\underset{n\to\infty}{\lim} \brn(h) = r(h) \Leftrightarrow \underset{n\to\infty}{\lim} \WF{\byn}{y} = 0.
\end{equation}
\end{proposition}
The proof can be found in the Appendix \ref{appendix:proof}. Note that the band-limited assumption is actually not needed for the left implication $[\Leftarrow]$.

%After defining the proposed WF distance and studying its properties, the following sections study the geometric and statistical features inherited from the Wasserstein distance for time series analysis.

%!TEX root = ../CRT_TSP2019_arxiv.tex

\section{A geodesic path between two times series}
\label{sec:geodesic}

We propose an interpolation framework for time series in the sense of \emph{travelling} from one time series, through the space of time series, to another time series. We clarify that by \emph{interpolation} we are not referring to data imputation of \emph{missing observations}. Though interpolation is common in  machine learning applications, to the best of our knowledge interpolation of time series is largely unexplored \CR{despite its possible impact on the development of novel regularisation techniques for signal processing, as audio smoothing for instance.}

\subsection{The need for spectral interpolation} % (fold)
\label{sub:need_interpolant}

Our proposed interpolation is based on the geodesic structure of the Wasserstein space \cite[Chap. 7]{ambrosio2008gradient}, which allows us to construct a geodesic trajectory, i.e., \emph{the shortest path}, between two time series---notice that the notion of ``shortest'' relies on the distance considered. Recall that for the usual Euclidean ($\L_2$) distance on the space of time series (functions of time with compact support), the trajectory between signals $x_1$ and $x_2$ consists in a $\L_2$ path that starts at $x_1$ and ends at $x_2$ by point-wise interpolation, that is, a superposition of signals of the form
\begin{equation}
   x_\gamma(t) =  \gamma x_1(t) + (1-\gamma)x_2(t),\qquad \gamma\in[0,1].
 \end{equation} 
This interpolation corresponds to  a spatial superposition of two signals from different sources (e.g., seismic waves), and is meaningful in settings such as source separation in Signal Processing, also referred to as the \emph{cocktail party} \cite{cherry1953some}. More generally, in general physical models that consider \emph{white noise}, there is an implicit assumption of additivity, where the spectral densities are linear combinations in the $\L_2$ (vertical) sense.

In other applications, however, such as audio (classification), body motions sensors, and other intrinsically-oscillatory processes, the superposition of time series fails to provide insight and understanding of the data. For instance, the $\L_2$ average of hand gestures from different experiments would probably convey little information about the true gesture and it is likely to quickly vanish due to the random phases (or temporal offsets). These applications, where spectral content is key, would benefit from a spectral interpolation rather than a temporal superposition. The need for a spectral-based interpolation method can be illustrated considering two sinusoidal time series of frequencies $\omega_1\geq\omega_2$, where the intuitive interpolation would be a sinusoidal series of a single frequency $w_\gamma$ such that $\omega_1\geq w_\gamma \geq\omega_2$. We will see that, while the $\L_2$ interpolation is unable to obtain such results, the proposed WF-based interpolation successfully does so.  

\subsection{Definition of the interpolant} % (fold)
\label{sub:def_interpolant}

Using the proposed WF distance and an extension of the McCann's interpolant, namely a constant-speed Wasserstein geodesic \cite[Theorem 7.2.2.]{ambrosio2008gradient}, we construct the WF-interpolant between time series  $x_1$ to $x_2$ as follows:
\begin{enumerate}
\item[(i)] embed signals $x_1$ and $x_2$ into the frequency domain by computing their NPSDs $s_1$ and $s_2$ according to eq.~\eqref{eq:def_npsd},
\item[(ii)] compute a constant-speed geodesic $(g_{\gamma})_{{\gamma}\in [0,1]}$ between $s_1$ and $s_2$ as
\begin{equation}
\label{def:geodesic}
g_{\gamma} = p_{\gamma}\#\pi^{\ast},
\end{equation}
where for $u,v\in\R,\ p_{\gamma}(u,v) = (1-\gamma)u+\gamma v \in\R$; $\pi^{\ast}$ is an optimal transport plan in eq.~\eqref{def:OT} between $s_1$ and $s_2$; and $\#$ is the pushforward operator. %\footnote{The expression $\nu=T\#\mu$ means that for any measurable set $B \subset\R^d$, we have $\nu(B)=\mu(T^{-1}(B))$.}
In other words, for any continuous function $h:\R\to\R$,
$$\int_{\R}h(u)\td g_{\gamma}(u) = \iint_{\R^2}h((1-\gamma)u+\gamma v)\td\pi^{\ast}(u,v),$$
and for each $\gamma\in [0,1], g_{\gamma}$ is a valid NPSD,
\item[(iii)] map $(g_{\gamma})_{\gamma\in [0,1]}$ back into the time domain to obtain the  time-series  interpolants $(x_{\gamma})_{\gamma\in [0,1]}$ between $x_1$ and $x_2$.
\end{enumerate}

Recall that the WF distance is a non-invertible embedding of time-series due to %the normalisation and to the fact that the phase is neglected. 
both the normalisation and the neglected phase. 
However, as we are interested in the harmonic content of the interpolant, we can equip the above procedure with a unit magnitude and either a fixed, random or zero phase depending on the application. In the case of audio signals, \cite{henderson2019audio} has proposed an interpolation path using the same rationale as described above and implemented it to create a portamento-like transition between the audio signals, where phase accumulation techniques are used to handle temporal incoherence.

In view of the proposed interpolation path, one can rightfully think of Dynamic Time Warping (DTW), a  widespread distance for time series \cite{bellman2015adaptive}, which consists in warping the trajectories in a non-linear form to match a target time series. As for the relationship  between WF and DTW, observe that  (i) the WF distance is rooted in frequency, whereas DTW is rooted in time, and (ii) though DTW does not lean on the Wasserstein distance, the authors of \cite{cuturi2017soft} proposed a Soft-DTW, using OT as a cost function, but again, purely in the time domain.

\subsection{Example: trajectories between parametric signals} % (fold)
\label{sub:sdeterministic_signals}

For illustration, we consider parametric (and deterministic) time series for which the proposed geodesic WF interpolants can be calculated analytically; this is due to their NPSDs and optimal transport plans being either known or straightforward to compute. For completeness, we focus on two complex-valued examples, yet they can be turned into real-valued ones with the appropriate choice of parameters. 

\textbf{Sinusoids.} Let us consider signals (of time) of the form 
\begin{equation}
  x_{a,b}(t)=e^{jat}+e^{jbt},  
\end{equation}
 with NPSD $s_{a,b}(\xi) =\tfrac{1}{2}(\delta(\xi-\frac{a}{2\pi})+ \delta(\xi-\frac{b}{2\pi}))$. For two signals $x_{a_1,b_1}$ and $x_{a_2,b_2}$  such that $a_1\leq a_2 <b_1\leq b_2$, we have $\WF{x_{a_1,b_1}}{x_{a_2,b_2}} =\sqrt{(a_1-a_2)^2+(b_1-b_2)^2}/(2\sqrt{2}\pi)$. Moreover, the interpolant between $x_{a_1,b_1}$ and $x_{a_2,b_2}$ for $\gamma\in [0,1]$ is given by 
\begin{align*}
x_{\gamma}(t) & =e^{j(\gamma a_1+(1-\gamma) a_2)t}+e^{j(\gamma b_1+(1-\gamma) b_2)t}\\
& =x_{\gamma a_1+(1-\gamma) a_2,\gamma b_1+(1-\gamma) b_2}.
\end{align*}
This example reveals a key feature of the WF distance: the WF interpolator \emph{moves} the energy across frequencies, whereas {vertical} distances perform a weighted average of the energy distribution, in the sense that computing the distance involves the integral of vertical displacements between the graphs of energy distribution. As a consequence, the WF interpolator is \emph{sparse} in frequency: the WF-interpolator between two line spectrum signals is also line spectrum.

\textbf{Exponentials.} Let us consider the following time series and its normalised power spectrum (which is Gaussian) respectively: 
\begin{align}  
x(t) &= e^{-\alpha t^2} e^{j2\pi t \mu}\label{eq:comp_exp}\\
S(\xi)  &= \NN(\mu,\alpha/4\pi^2),
\end{align}
where $(\mu,\alpha)\in\R\times\R_+^{\ast}$ are parameters. The Wasserstein distance and the geodesic between two Gaussians are known, in fact, the interpolant in this case is also a Gaussian \cite{takatsu2011wasserstein}. Therefore, the WF distance and interpolant between two complex exponentials $x_1$ and $x_2$ with parameters $(\mu_1,\alpha_1)$ and $(\mu_2,\alpha_2)$ are respectively given by:
\begin{align*}
\WF{x_1}{x_2} &= \sqrt{(\mu_1-\mu_2)^2+(\sqrt{\alpha_1} - \sqrt{\alpha_2})^2/4\pi^2},\\
x_{\gamma}(t)&=e^{-[\gamma\sqrt{\alpha_1}+(1-\gamma)\sqrt{\alpha_2}\ ]^2t^2} e^{j2\pi t (\gamma\mu_1 +(1- \gamma) \mu_2)}.
\end{align*}
Again, we see that the proposed interpolant preserves the nature of the signals: the WF-interpolation between complex exponentials of the form in eq.~\eqref{eq:comp_exp} is also a complex exponential. Fig.~\ref{fig:gauss_path} shows the interpolation between two complex exponentials via 6 evenly-spaced (according to WF) time series along the geodesic path; only the real part is shown for clarity.
\begin{figure}[!ht]
\centering
\hspace{-1.3em}\includegraphics[width=0.2\textwidth]{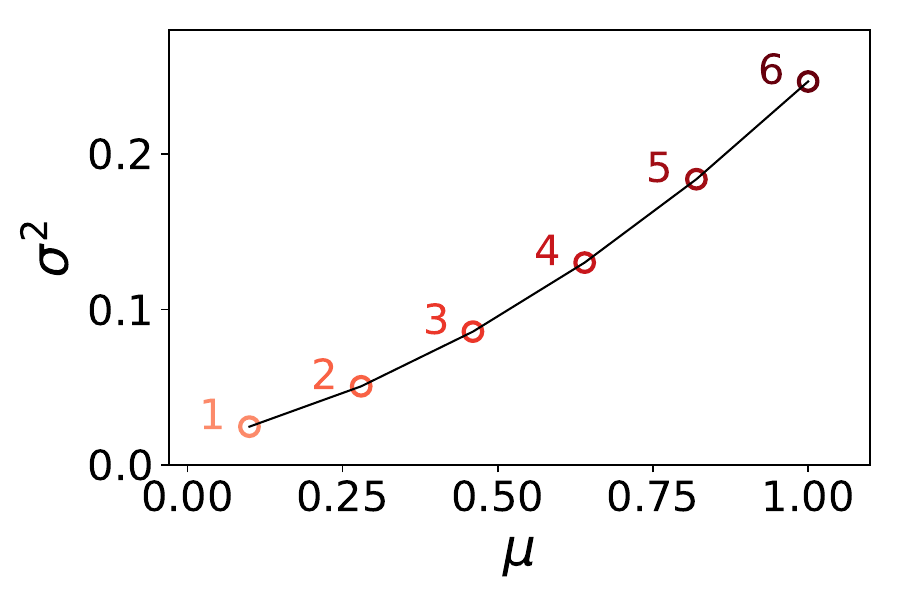}\\
\includegraphics[width=0.48\textwidth]{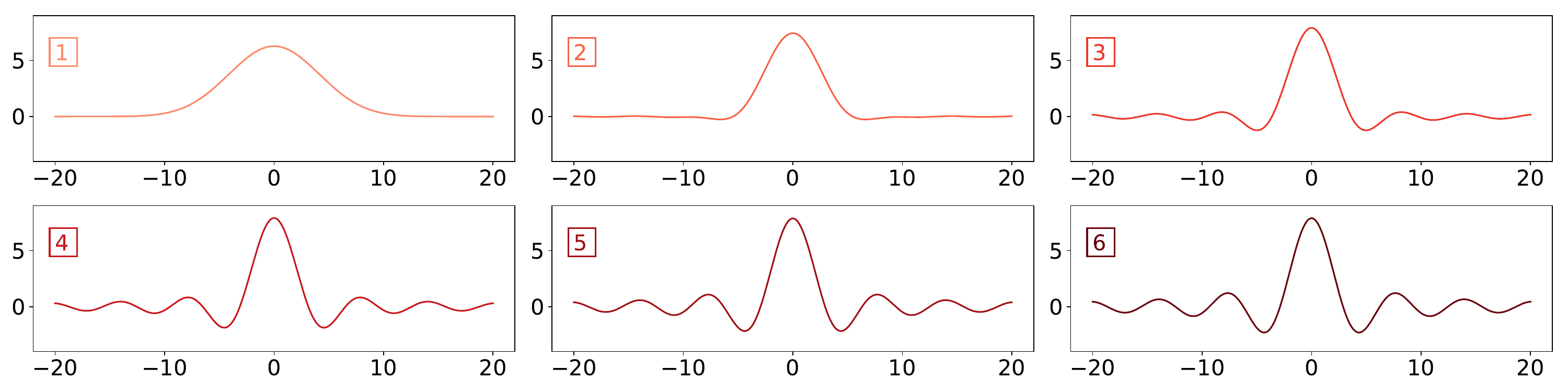}
  \caption{\textbf{Top:} The geodesic path between Gaussian NPSDs in the parameter space where $\sigma^2 = \frac{\alpha}{4\pi^2}$ is the variance and $\mu$ is the mean. The initial NPSD has parameters $\sigma^2_1=0.025$ and $\mu_1=0.1$ and the final NPSD has parameters $\sigma^2_2=0.25$ and $\mu_2=1$. \textbf{Bottom:} Time-domain signals (real part) corresponding to each of the 6 PSDs in the geodesic path.}
  \label{fig:gauss_path}
\end{figure}
 
\subsection{Example: data augmentation for the  {\em C. elegans} dataset}

The proposed interpolation method can be used for data augmentation by computing synthetic time series along the WF-interpolation path between two (real-world) signals. We next illustrate this property with the \emph{worm dataset} \cite{UCRTime}, also referred to as {\em C. elegans} behavioural database \cite{bhatla2009wormweb}. Since the movements of {\em C. elegans} can be represented by combinations of four base shapes (or eigenworms) \cite{brown2013dictionary}, a worm movement of the present dataset can be decomposed into these bases.
% amplitudes (in one dimension) of the worm along the first component of the base shapes. 
We will only consider the class of wildtype worms on the first base shape, for which the one-dimensional time series are 900-sample long. Our aim is to construct a WF-interpolation path between the movements of two worms belonging to the same class to synthetically generate more signals having the same spectral properties. %of the same type class having the same spectral properties.

Recall that to compute the time series from the interpolated NPSDs we need the phase, for this experiment we will interpolate the phases too in an Euclidean way: for two phases $\phi_1$ and $\phi_2$ we choose the interpolator  $\phi_{\gamma} := \gamma(\phi_1 \operatorname{mod} 2\pi)+(1-\gamma)(\phi_2\operatorname{mod} 2\pi)$. Furthermore, we will consider an evenly-spaced frequency grid of $1001$ points between $-40$ and $40$. Fig.~\ref{fig:worms_path} shows a 10-step interpolation path between two {\em C. elegans} series, using the WF (top) and the Euclidean (bottom) distances.

\begin{figure}[!ht]
\centering
\includegraphics[height=12em]{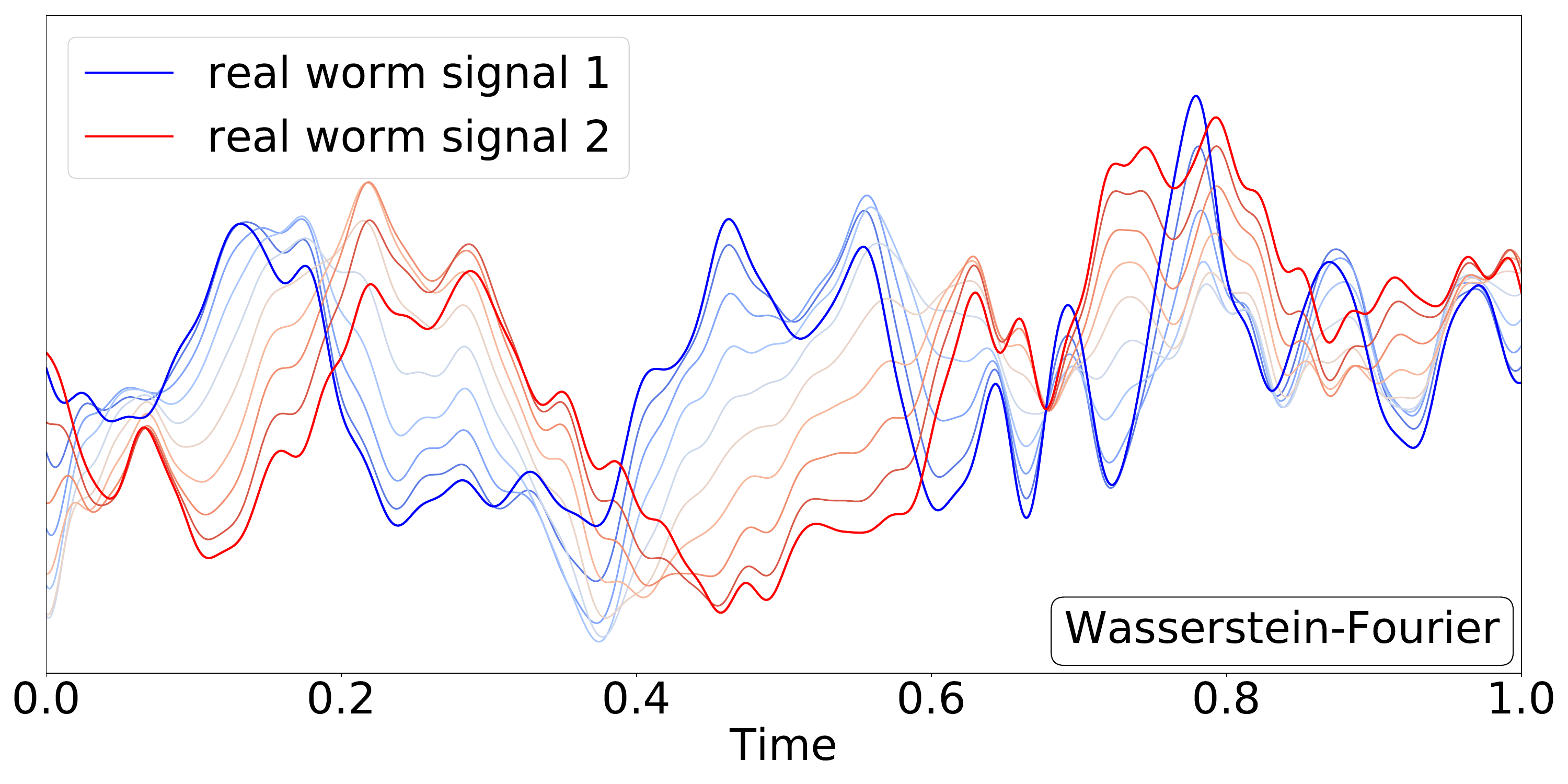}
\includegraphics[height=12em]{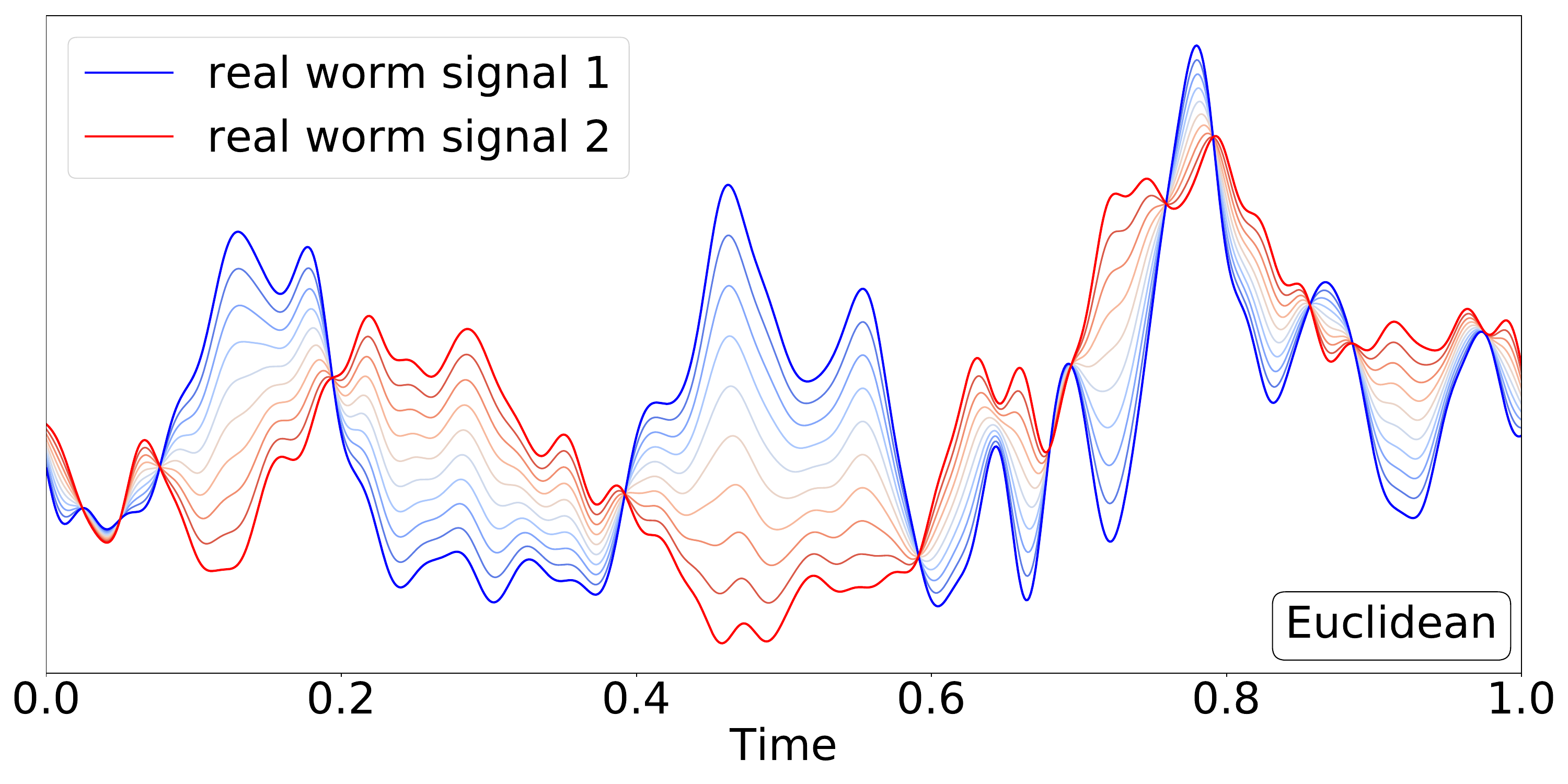}
  \caption{10-step interpolation $(x_{\gamma})_{\gamma\in [0,1]}$ between two signals from the {\em C.~elegans} database using the proposed WF distance (top) and the Euclidean distance (bottom): the true signals are shown in solid blue and red, while the interpolations are colour-coded with respect to $\gamma$.}
  \label{fig:worms_path}
\end{figure}

\begin{remark}
Observe that the WF interpolation is not constrained to remain in the element-wise convex hull spanned by the two series. On the contrary, under the Euclidean distance if the two series coincide in time and value, then all the interpolations are constrained to pass through that value, e.g., in Fig.~\ref{fig:worms_path} (bottom) for $t=0.2$ and many other time stamps.
\end{remark}

Additionally, to assess how representative our interpolation is in the class of wildtype worms, let us refer to the distance between a NPSD $s$ and a class $C$ as 
\begin{equation}
   d(s,C) = \min_{s'\in C} d(s,s'), \label{eq:class_dist}
 \end{equation}
and verify that the Wasserstein interpolant is closer to the wildtype class NPSDs than its Euclidean counterpart. Indeed, Fig.~\ref{fig:worms_path_behavior} shows our WF-interpolant (for $\gamma=0.55$) and the signal $s'$ minimising the above distance \eqref{eq:class_dist} for comparison. Besides the visual similarity, notice from Table \ref{tab:distances} that the WF interpolant is much closer to the wildtype class (using both the WF and Euclidean distances) than the Euclidean interpolant. This notion of class proximity for time series based on NPSDs opens new avenues for time series classification as we will see in Sec.~\ref{sec:log_reg}.

\begin{figure}[!ht]
\centering
\includegraphics[height=12em]{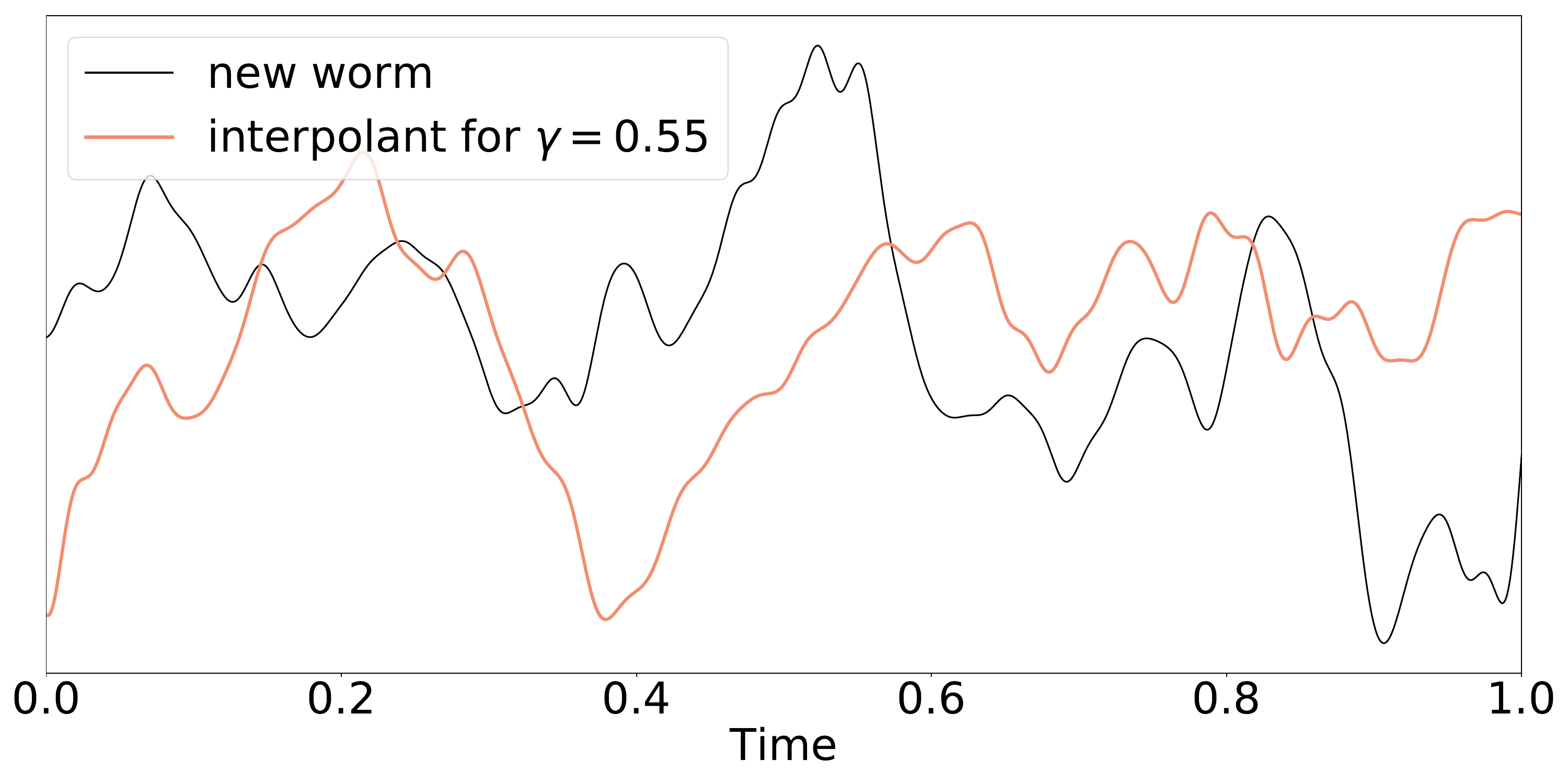}
  \caption{WF interpolant and minimiser of eq.~\eqref{eq:class_dist}: A wildtype worm signal unseen before the interpolation (black) and the proposed WF interpolant for $\gamma=0.55$ (light red).}
  \label{fig:worms_path_behavior}
\end{figure}

\begin{table}[!ht]
\centering
\caption{Class distances (in frequency domain) as defined in eq.~\eqref{eq:class_dist} for the WF-interpolant (left) and the Euclidean interpolant (right).}
\small
\begin{tabular}{p{0.5cm}cc}\toprule
               &  WF interpolant          &  Euclidean interpolant        \\ \toprule
$\L_2$         & $\textbf{3.9587.1e-4}$  &  $3.9905.1e-4$ \\
$W_2$    & $\textbf{0.1135}$ &     $0.2431$
\end{tabular}

\label{tab:distances}
\end{table}

\subsection{Special case of Gaussian processes} 
\label{ssub:gps}

The Gaussian process (GP) prior is a non-parametric generative model for continuous-time signals \cite{Rasmussen:2006}, which is defined in terms of its covariance or, via Bochner theorem \eqref{eq:Wiener}, by its PSD. Following from the arguments in Sec.~\ref{sec:class}, we can identify an equivalence class of all the GPs that share the same covariance function (and hence PSD) up to a scaling constant. Identifying this equivalence class is key to extend the WF distance to GPs, which is in line with the state of the art on GPs regarding spectral covariances \cite{Wilson:2013,parra_tobar,CSM,tobar:nonparametric,tobar19,hensman2016variational,lazaro2010sparse,tobar19b}. 

The proposed interpolation is in fact a geodesic in the GP space (w.r.t. to the WF distance), where every point in the trajectory is a valid GP. The interpolation can be implemented in the space of covariance kernels simply by applying the inverse Fourier transform to the proposed WF interpolator on NPSDs, where the phase is no longer needed since covariances are even functions. In this way,   as the properties of a GP are determined by its kernel (e.g., periodicity, differentiability, long-term correlations, and band-limitedness), the WF interpolant is a transition in the space of GPs with these properties. Fig.~\ref{fig:gp_interpol} illustrates this interpolation for four different covariance kernels, notice how the interpolation results on smooth incorporation of frequency components.

\begin{figure}[!ht]
  %\hspace{-1.7cm}
  %\centering
  \includegraphics[scale=0.55]{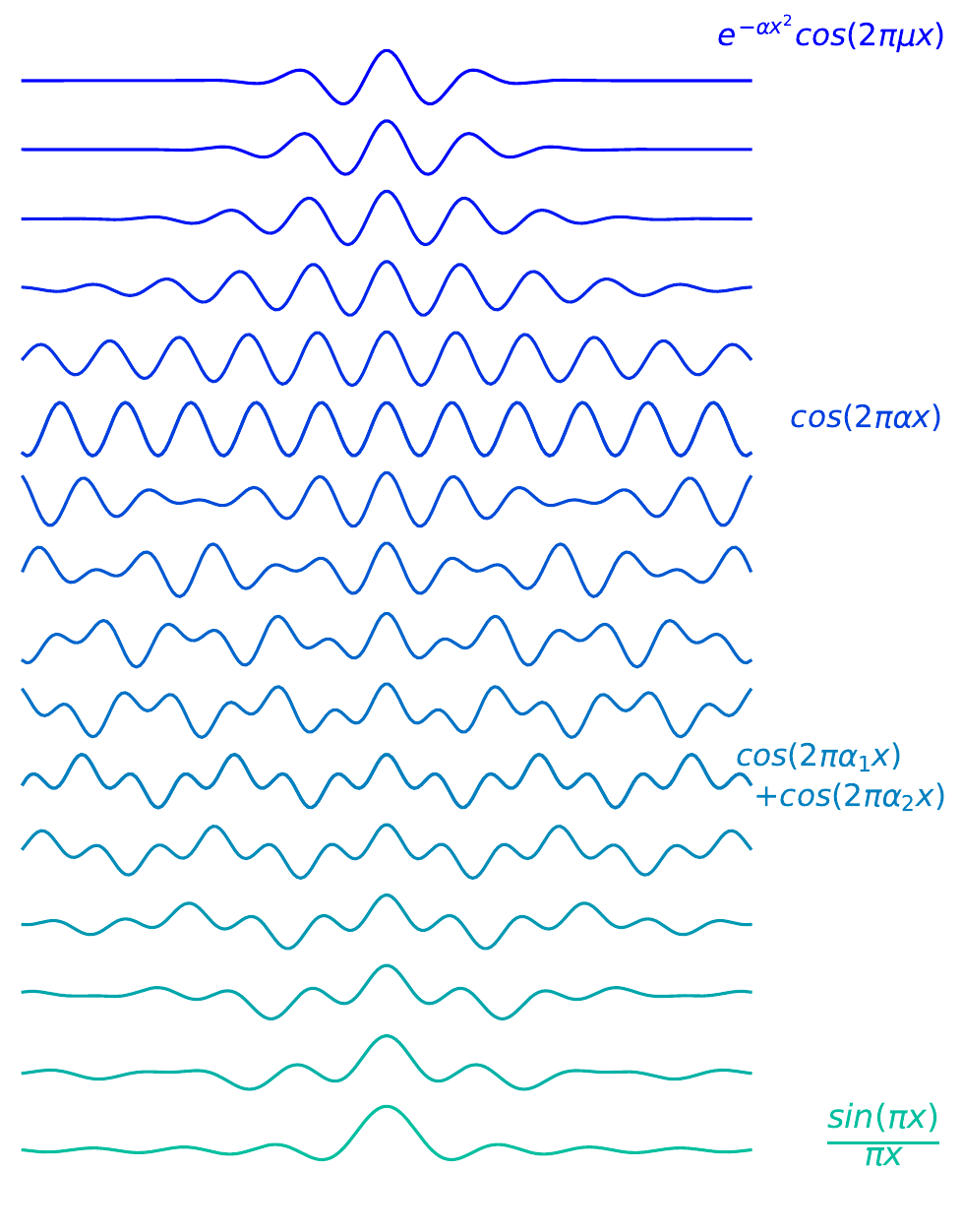}
  \caption{Interpolation of four covariance kernels colour coded: spectral mixture (blue, \cite{Wilson:2013}), cosine (violet), cosine mixture (turquoise), and sinc (green). Between two given kernels, we display $4$ interpolants for $\gamma\in\{0.2, 0.4, 0.6, 0.8\}$. Optimal transport plans $\pi^{\ast}$ were computed using POT library \cite{flamary2017pot}.}
  % \caption{Each column shows a WF interpolation between two GP families from top to bottom. Kernels are colour coded: spectral mixture (blue, \cite{Wilson:2013}), cosine (orange), cosine mixture (green), and sinc (red). Interpolants are indexed on the left-hand side ($\gamma$) and combine the corresponding colour codes, notice how the interpolation results on smooth incorporation of frequency components. Optimal transport plans $\pi^{\ast}$ were computed using POT library \cite{flamary2017pot}.}
  \label{fig:gp_interpol}
\end{figure}

\iffalse
\subsection{Training a GP on a Fourier-Wasserstein loss}

Training a GP is difficult and expensive. However, we know that for a multivariate Gaussian, the maximum likelihood estimator of the covariance/autocorr function is $\frac{1}{n-\delta}\sum_{i=1}^{n-\delta} x(i)x (i-\delta)$, therefore, we can say something like the ML estimate of the PSD is the square discrete Fourier transform (apply the FT to the previous expression). Therefore, we can train a GP by first computing PSDs (via DFT or FFT or Welch) and use barycenters etc, to then choose a parametric PSD that is closest to our computed sampled PSD. Then, wee anti-Fourier transform such parametric PSD and use it as a kernel of the GP. It is intuition, but we can justify it with connection with ML as I stated in the beginning of this paragraph.

\fi

%!TEX root = ../CRT_TSP2019_arxiv.tex

\section{Principal geodesic analysis in WF}
\label{sec:PCA}

In order to detect spectral patterns across multiple time series, we can apply dimensionality reduction techniques to their respective NPSDs. Dimensionality reduction allows for analysing the main modes of variability of a dataset by basically finding a sequence of subspaces minimising the sum---over the data---of the norms of projection residuals. Though the \emph{de facto} method for dimensionality reduction applied to functions is Functional Principal Component Analysis (FPCA) \cite{dauxois1982asymptotic}, applying it to NPSDs provides limited interpretation. This is because using the Euclidean distance (assumed by FPCA) directly on the NPSDs might result on principal components (PC) with negative values, which are not distributions and thus hinder interpretablity. %that are not distributions since they could have negative values. 
To circumvent this challenge, the Wasserstein distance can be used to extend the classical FPCA, which takes place in the Hilbert space $\L_2(\R)$, to a Principal Geodesic Analysis (PGA), which takes place in the space of densities. In our particular setting, this will allow us to perform PGA for NPSDs in the Wasserstein space.

In a nutshell, performing PGA for probability distributions in the Wasserstein space corresponds to considering a linear space tangent to the Wasserstein one, and carrying out a standard PCA in this linear space. Then, the principal components are projected back to the Wasserstein space. This procedure is referred as Principal Geodesic Analysis (PGA). For the reader's convenience, a review of this approach is presented in Appendix \ref{appendix:PCA}.

%\FE{can we please here very briefly (6-8 lines) describe the main concept as I moved the \emph{Description of the general procedure} to the Appendix. I am thinking something in the likes of: ``We cannot perform PCA directly on the Wasserstein space, so we'll consider a linear space which is tangent to the wasserstein space to prform PCA and then project back to the Wasserstein space. We refer to this procedure as Principal Geodesic Analysis (PGA) For the reader's convenience, a review of this procedure is presented in Appendix \ref{appendix:PCA}''}

\subsection{Implementation on NPSDs: synthetic and real-world data}
\label{sub:PGAimplementation}

The PGA procedure can be applied to the space of NPSD in order to yield a meaningful decomposition of the signals' main oscillatory modes and to summarise the information of the series into geodesic components. \CR{The algorithm to compute PGA in the Wasserstein-Fourier space is presented in Appendix \ref{appendix:PCA_algo}}. We illustrate this application via a toy example using a dataset composed of:
\begin{itemize}
\item Group 1: $25$ cosine signals with random frequency $f_i \sim \operatorname{unif}(1,5)$ and white Gaussian noise.
\item Group 2: $25$ sinc signals with random parameter $a_i  \sim \operatorname{unif}(1,5)$ and white Gaussian noise.
\end{itemize}
Fig.~\ref{fig:signal_cos_sinc} shows $3$ signals from each group. The cosines are colour-coded in red, while the sincs are shown in yellow. Then, Fig.~\ref{fig:cos_sinc_proj} shows the projection locations $t_{1,n}\in[-1.52, 2.51]$ of all $n=50$ NPSDs along the first geodesic component $g_1$. In particular, the locations of the projections $t_{1,n}$ allow for a linear separation of the groups 1 and 2 since they are respectively in the ranges $t_{1,n}\in[-0.36, 2.51]$ (red-coded, cosines), and $t_{1,n}\in[-1.52, -0.75]$ (yellow-coded, sinc), meaning that the group of a new sample can be identified by assessing if its first component projection is greater or less than, say, $\bar{t} = -0.5$.

\begin{figure}[!ht]
	\centering
		\hspace{-0.5cm}\includegraphics[height=13em]{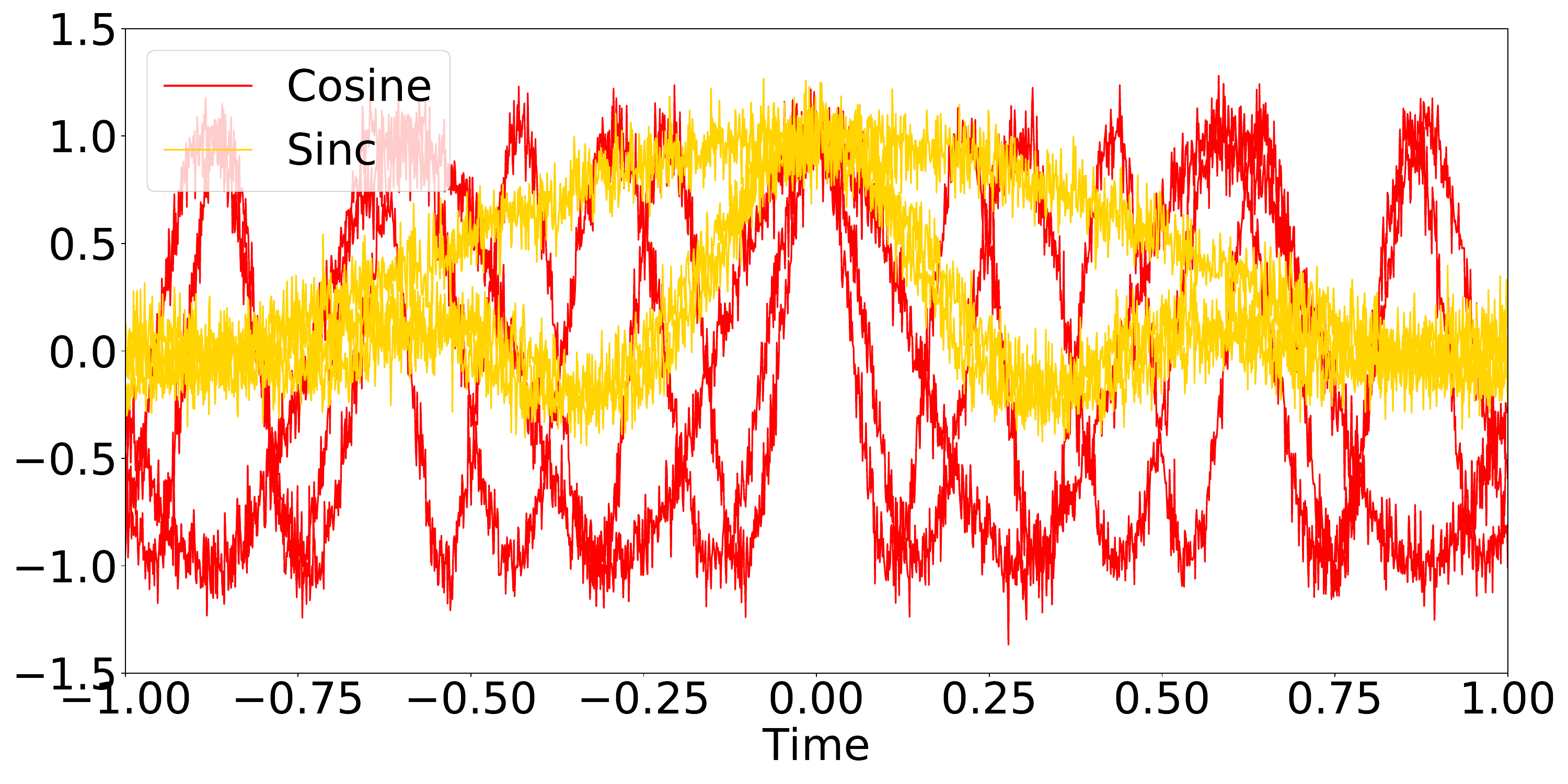}
	\caption{Subsamples of $3$ out of $25$ noisy cosine signals (red) (resp. noisy sinc signals (yellow)).}
	\label{fig:signal_cos_sinc}
\end{figure}

\begin{figure}[!ht]
	\centering
\includegraphics[height=8em]{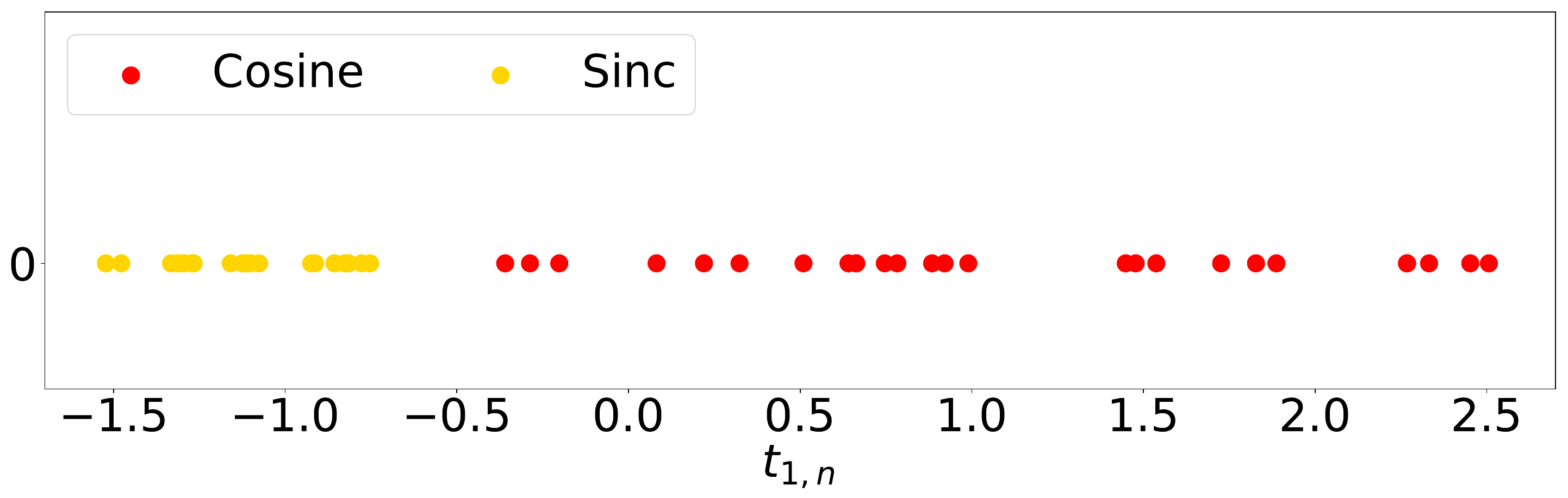}
	\caption{Values of the projection locations $(t_{1,n})_{n=1,\ldots,50}$ of the NPSDs onto the first geodesic component $g_1$ for the cosine (red) and sinc (yellow) signals. The groups can be linearly separated only looking at this projection.}
	\label{fig:cos_sinc_proj}
\end{figure}

To validate the usefulness of the proposed procedure, we also applied FPCA (i) directly on the signals, and (ii) on the NPSDs. In neither of these cases it was possible to discriminate between the sinc or cosine groups.

Lastly, we applied this decomposition procedure on the real-world \emph{fungi dataset} composed of one-dimensional melt curves of the rDNA internal transcribed spacer (ITS) region of $51$ strains of $18$ fungal species (\cite{lu2017dynamic, UCRTime}). Each fungi species contains between $7$ and $19$ signals, for a total of $189$ time series of length $201$. The dataset is displayed in Fig.~\ref{fig:fungies}, where the $18$ species are colour coded.

\begin{figure}[!ht]
	\centering
	\includegraphics[height=12em]{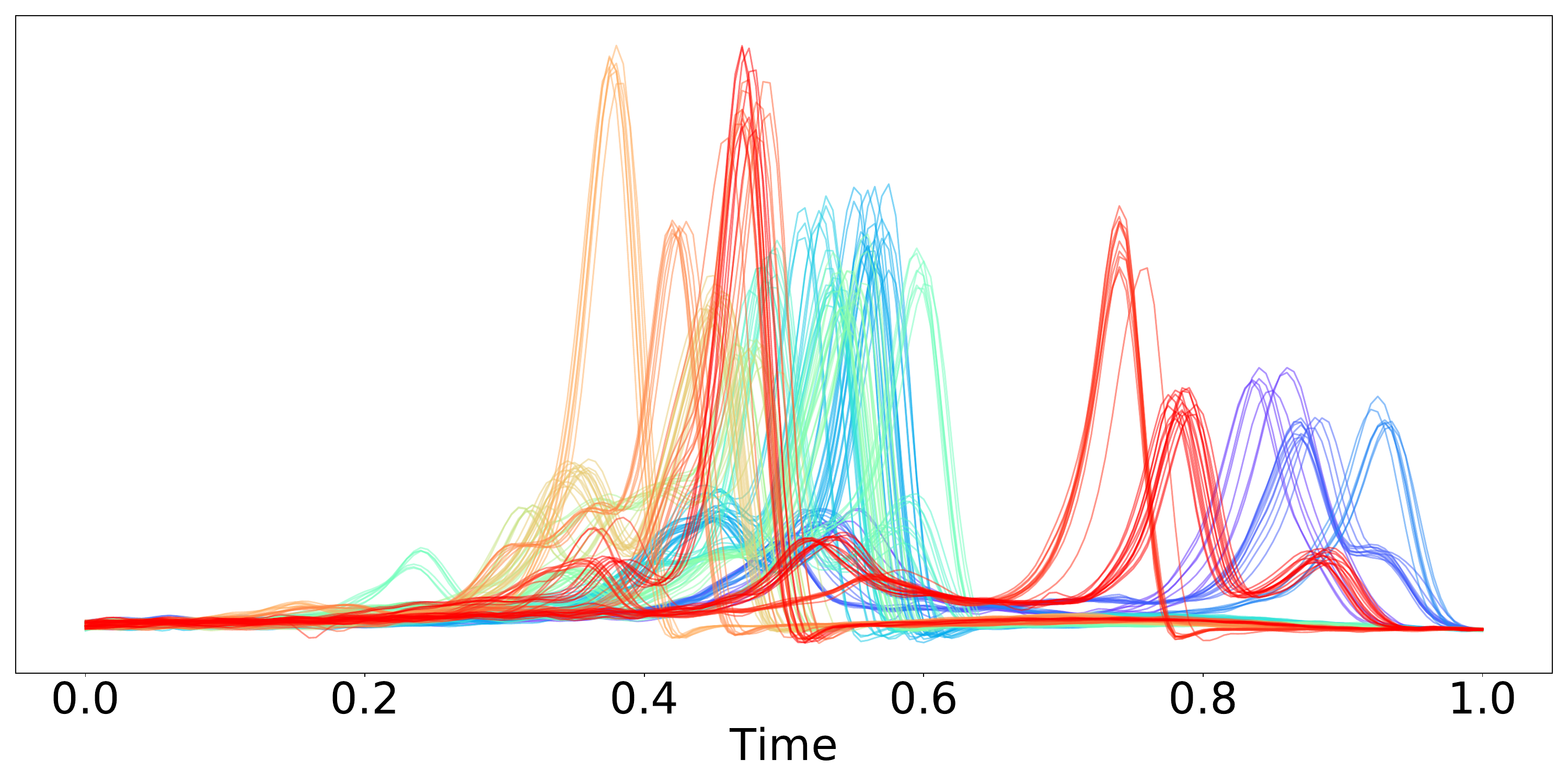}
	\caption{Melt curves for $18$ colour-coded species of fungies.}
	\label{fig:fungies}
\end{figure}

Fig.~\ref{fig:fungies_proj} shows the locations of projections  %$\{t_{1,n}\}_{n=1}^{189}$ for the first component against the class index (top), and the values
$\{t_{1,n},t_{2,n}\}_{n=1}^{189}$ for the first and second components plotted against one another using the same colour-codes are the species in Fig.~\ref{fig:fungies}. Notice that some regions contain only projections of a single species while some species overlap (at least for the first two components), therefore, the proposed PGA might suggest that some groups have unique features isolated in different clusters while others share common properties.

\begin{figure}[!ht]
\centering
		\includegraphics[height=13em]{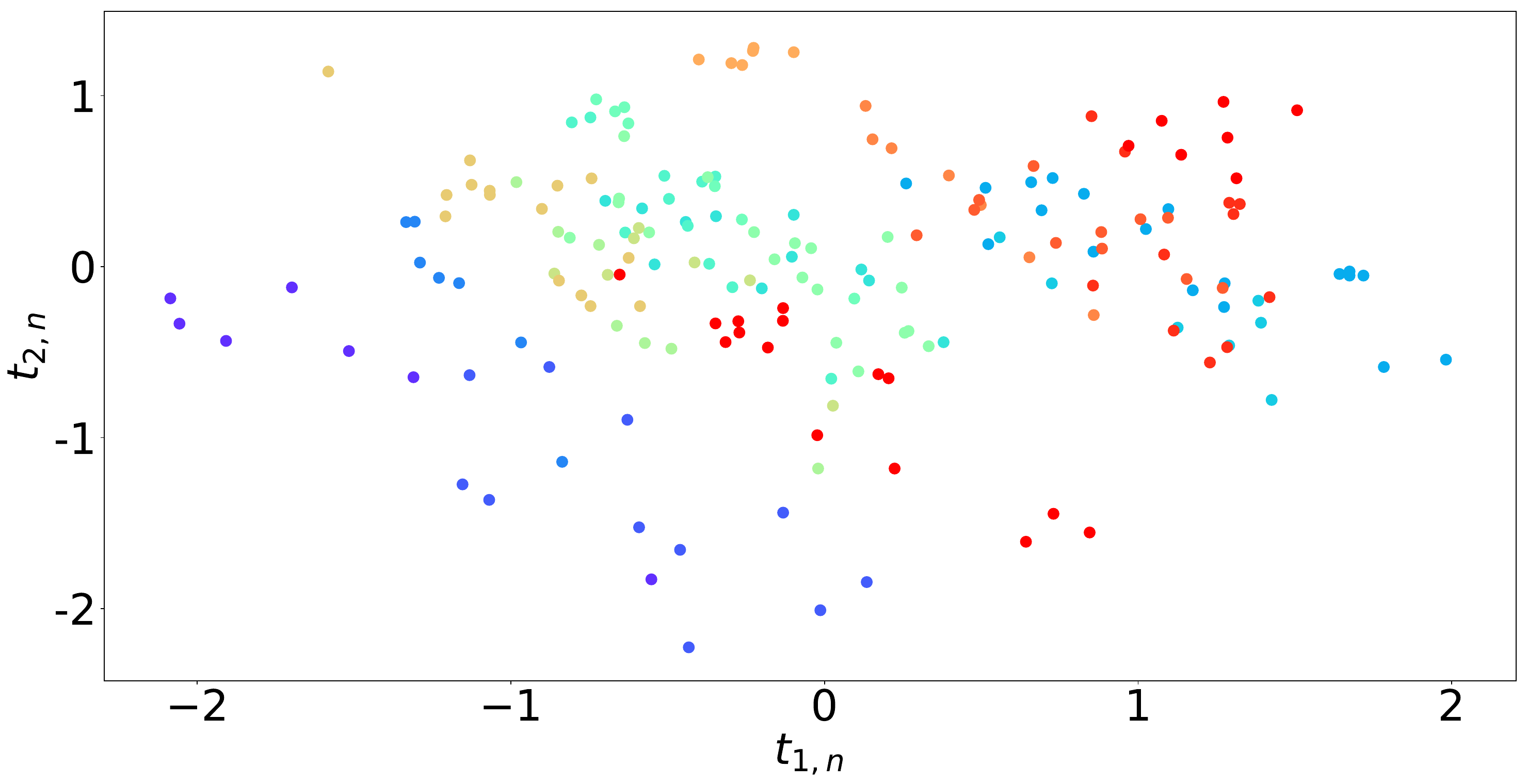}
	\caption{Values of the projection locations $(t_{1,n},t_{2,n})_{n=1,\ldots,189}$. The colours encode each fungies species, where it can be seen that some of them overlap and some do not.}
	\label{fig:fungies_proj}
\end{figure}

%!TEX root = ../CRT_TSP2019_arxiv.tex

\section{Classification of time series using divergences}
\label{sec:log_reg}
  
This last section is dedicated to evaluating the proposed WF distance, as well as other divergences on PSDs, as a metric for classifying time series. Note that a similar idea for a regression model has been developed in the special case of biological data in \cite{martinez2016closed}. We start with a toy example motivating the use of the WF distance in the classification task, to then propose classifiers using four NPSD distances (WF, KL, IS and Euclidean) combined with the logistic function (parametric) and the $k$-nearest neighbours concept (non-parametric). Then, these classifiers are validated on binary classification using two real-world datasets.

\subsection{Motivation for spectrum-based classification}
 
Let us consider two classes of synthetic NPSDs given by 
\begin{itemize}
	\item \emph{Left-Asymetric Gaussian Mixture (L-AGM):} given by a sum of two Gaussians with random means and variances, where the left Gaussian has a variance four times that of the right one.
	\item \emph{Right-Asymmetric Gaussian Mixture (R-AGM):} constructed in the same manner but the role of the variances is reversed.
\end{itemize} 
In both classes the distance between the means remains constant. Using both the Wasserstein and Euclidean metrics Fig.~\ref{fig:log_toy} shows samples of each class, their means or \emph{barycenters}, and the distances between the samples and the means. 

\begin{figure}[!ht]
	\centering
	\includegraphics[width=0.48\textwidth]{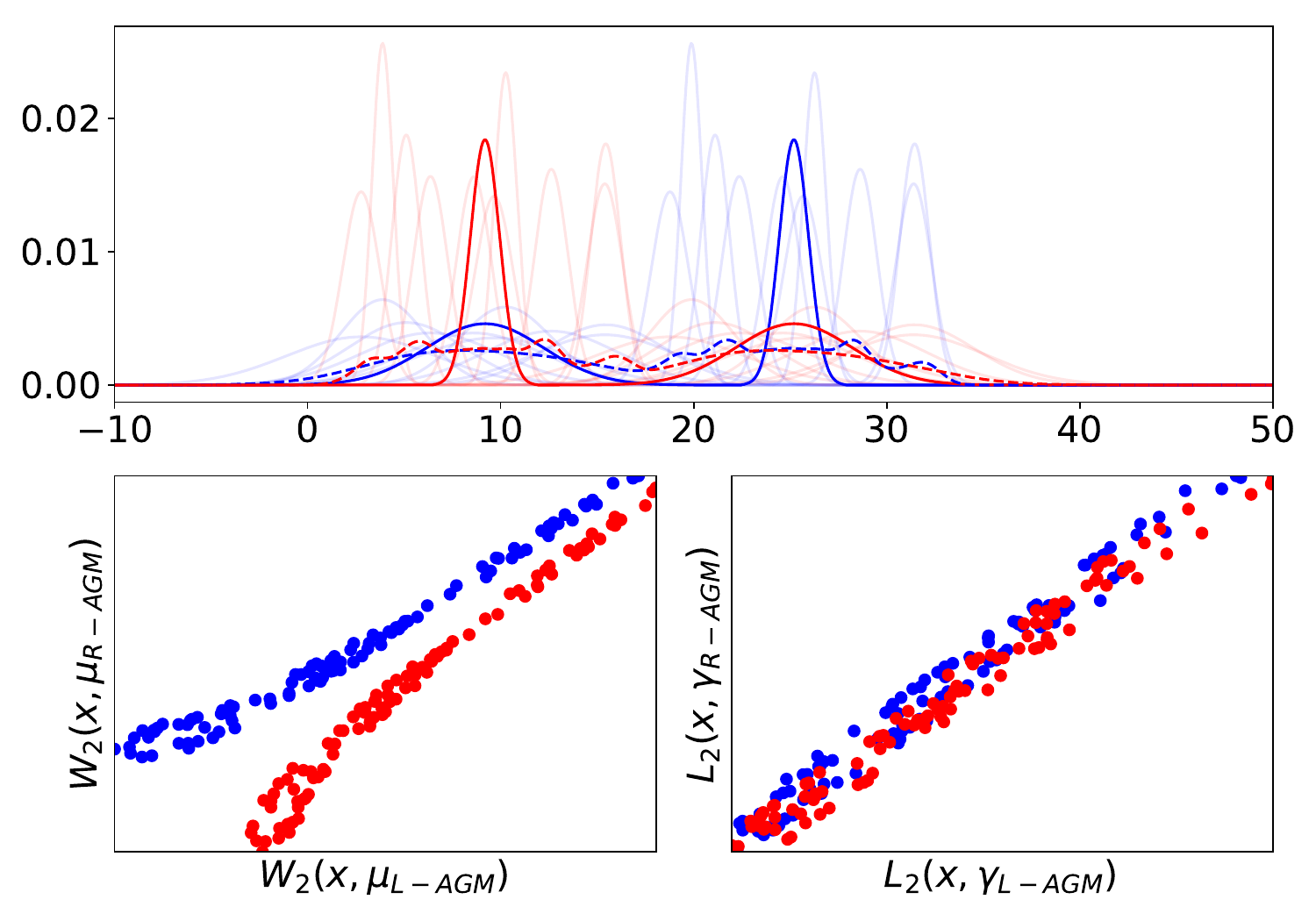}
	\caption[Toy example] 
	{Illustration of the linear separability of distributions using $W_2$. \textbf{Top}: $100$ samples from each class L-AGM (light blue) and R-AGM (light red),  with their Euclidean (dashed line) and Wasserstein barycenters (solid line). \textbf{Bottom}: Distance to the barycenter of R-AGM plotted against distance to the barycenter of L-AGM (left: Wasserstein, right: Euclidean), colour-coded by class. }
	\label{fig:log_toy}
\end{figure}

\begin{remark}\label{rem:}
Only by measuring the Wasserstein distance between a sample and the class barycenters, the sample can be appropriately assigned to the corresponding class. Using the Euclidean distance, on the contrary, does not provide the required separability.
\end{remark}

\subsection{Definition of spectral classifiers for time series } 
\label{sec:TS_classification}

We propose classifiers based on the logistic function and a $k$-nearest neighbours (KNN) classifier, both using different divergences for NPSDs as initially addressed by \cite{rakotomamonjy2018distance}.\\

\begin{table*}[t!] 
\centering
\small
\caption[labelsep = newline]{Logistic Classification: Accuracy and standard error for \textit{gun shot} against the 9 remaining classes of Urbansound8k. For each class, the best and second best performances are shown in boldface black and blue respectively for all KNN and logistic regression models. The last column shows the performance of $\mathcal{L}_{ALL}$ (a logistic regressor that combines $\L_2, \KL$ and $W_2$ features) which is highlighted in red when it outperformed all 4 other single-distance logistic regressors.}
\label{tab:urbansound_lr}
\scalebox{0.83}{
	\begin{tabular}{p{1.2cm}p{1.8cm}p{1.8cm}p{1.8cm}p{1.8cm}p{1.8cm}p{1.8cm}p{1.8cm}p{1.8cm}|p{1.8cm}} \toprule

		& $\kNN_{W_2}$ & $\kNN_{\L_2}$ & $\kNN_{\KL}$ & $\kNN_{IS}$ &  $\mathcal{L}_{W_2}$ & $\mathcal{L}_{\L_2}$ & $\mathcal{L}_{\KL}$ & $\mathcal{L}_{IS}$ & $\mathcal{L}_{ALL}$ \\ \midrule  

		AC &  $0.734 (\pm 0.065)$ & $0.634 (\pm 0.087)$ & $0.611 (\pm 0.087)$ & $\bb{0.769} (\pm 0.03)$ & $0.732 (\pm0.072)$ & $0.718 (\pm 0.047)$ & $0.650 (\pm0.090)$  &  $\textbf{0.795} (\pm 0.036)$ & $0.672 (\pm0.088)$ \\
 
		horn & $0.755 (\pm 0.046)$ & $0.783 (\pm 0.037)$ & $\textbf{0.809} (\pm 0.019)$ & $0.681 (\pm 0.042)$  & $0.588 (\pm0.077)$ & $0.743 (\pm0.043)$ & $\bb{0.790} (\pm0.037)$  &  $0.632 (\pm 0.073)$ & $\br{0.829} (\pm 0.031)$  \\
 
		children & $0.747 (\pm 0.023)$ & $0.6 (\pm 0.028)$ & $0.672 (\pm 0.014)$ & $0.708 (\pm 0.015)$  & $\bb{0.751} (\pm0.027)$ & $0.685 (\pm0.031)$ & $0.736 (\pm0.023)$ & $\textbf{0.756} (\pm 0.028)$ & $\br{0.769} (\pm 0.031)$ \\ \midrule

		dog bark & $0.721 (\pm 0.033)$ & $0.693 (\pm 0.026)$ & $0.658 (\pm 0.028)$ & $\textbf{0.746} (\pm 0.03)$ & $\bb{0.743} (\pm0.040)$ & $0.720 (\pm0.033)$ & $0.728 (\pm0.040)$ & $0.738 (\pm 0.044)$ & $0.724 (\pm 0.038)$  \\ 

		drilling & $0.82 (\pm 0.019)$ & $0.76 (\pm 0.029)$ & $0.747 (\pm 0.022)$ & $0.756 (\pm 0.02)$  & $\bb{0.827} (\pm0.027)$ & $0.826 (\pm0.026)$ & $0.817 (\pm0.026)$ & $\textbf{0.844} (\pm 0.023)$ & $0.836 (\pm0.020)$   \\ 

		EI & $\textbf{0.816} (\pm 0.029)$ & $0.774 (\pm 0.035)$ & $0.794 (\pm 0.032)$ & $0.791 (\pm 0.038)$  & $0.767 (\pm0.041)$ & $0.733 (\pm0.051)$ & $\bb{0.791} (\pm0.042)$ & $0.732 (\pm 0.05)$ & $0.791 (\pm0.045)$  \\ \midrule

		J & $\textbf{0.791} (\pm 0.044)$ & $0.617 (\pm 0.049)$ & $\bb{0.737} (\pm 0.049)$ & $0.687 (\pm 0.045)$  & $0.645 (\pm0.087)$ & $0.585 (\pm0.095)$ & $0.669 (\pm0.059)$ & $0.388 (\pm 0.099)$ & $ \br{0.709} (\pm0.06)$  \\

		siren & $\textbf{0.924} (\pm 0.025)$ & $0.868 (\pm 0.038)$ & $0.883 (\pm 0.038)$ & $0.78 (\pm0.054)$  & $0.803 (\pm0.062)$ & $0.878 (\pm0.034)$ & $\bb{0.897} (\pm0.034)$ & $0.454 (\pm 0.078)$ & $\br{0.913} (\pm0.027)$  \\
 
		music & $\textbf{0.819} (\pm 0.021)$ & $0.769 (\pm 0.017)$ & $0.763 (\pm 0.023)$ & $\bb{0.814} (\pm0.028)$ & $0.792 (\pm0.030)$ & $0.782 (\pm0.025)$ & $0.812 (\pm0.029)$ &  $0.284 (\pm 0.67)$ & $ \br{0.845} (\pm0.028)$  \\ \bottomrule

\end{tabular}
}
\vspace{1em}
\end{table*}

\noindent\textbf{Logistic classification for NPSDs.} In binary classification (i.e., classes  $C_0$ and $C_1$), we can parametrise the conditional probability that the sample $s$ is class $C_1$ using the logistic function:
\begin{equation}
\label{eq:reparametrisation}
p(C_0\vert s) = \frac{1}{1+e^{-\alpha+\beta d(s,\bar{s}_0)- \gamma d(s,\bar{s}_1)}},
\end{equation}
where the function $d(\cdot, \cdot)$ is a divergence and $\bar{s}_0$ (resp. $\bar{s}_1$) is known as the \emph{prototype} of the class $C_0$ (resp. $C_1$). For a set of sample-label observation pairs $\{(s_i,C_i)\}_{i=1}^N$, the set of parameters $[\alpha, \beta, \gamma]$ can be obtained via, e.g., maximum likelihood. Notice that, in connection with the standard (linear) logistic regression, the argument in the exponential in eq.~\eqref{eq:reparametrisation} is a linear combination of the distances to class prototypes; therefore, relative proximity (on distance $d$) to the prototypes determines the class probabilities in a linear fashion.  

As the sample $s$ in our setting is an NPSD, we consider four alternatives for $d$ and $\{\bar{s}_k\}_{k=0,1}$ in eq.~\eqref{eq:reparametrisation}, thus yielding four different classifiers:
\begin{enumerate}
\item[(1)] Model $\LL_{\L_2}$ : $d := \L_2$ the Euclidean distance, and $\bar{s}_k :=\frac{1}{\# C_k} \sum_{i \in C_k} x_i$, for $k \in \{0,1\}$, is given as the Euclidean mean of class $k$,
\item[(2)] Model $\LL_{\KL}$ : $d := \KL$ the Kullback-Leibler distance, and $\bar{s}_k$ is also given as the Euclidean mean of class $k$,
\item[(3)] Model $\LL_{\IS}$ : $d := \IS$ the Itakura-Saito distance, and $\bar{s}_k$ is also given as the Euclidean mean of class $k$,
\item[(4)] Model $\LL_{W_2}$ : $d := W_2$ the Wasserstein distance, and $\bar{s}_k$ is given as the Wasserstein mean \eqref{def:barycenter} of class $k$.
\end{enumerate}

Also, as the KL and IS divergences are not symmetric, we clarify its use within models $\LL_{\KL}$ and  $\LL_{\IS}$ via the following remark.
\begin{remark}
The support of the Euclidean barycenter $\bar{s}$ of a family of distributions $\{s_i\}_i$ contains the support of each  member $s_i$. Therefore, the KL (resp. IS) divergence computed in the following direction $\KL(s\Vert \bar{s})$ (resp. $\IS(s\Vert \bar{s})$) where $s$ is any NPSD, is always finite and thus to be considered within model $\LL_{\KL}$ (resp. $\LL_{\IS}$).
\end{remark}

\noindent\textbf{KNN classification for NPSDs.} We also consider a variant of the KNN classifier, using the aforementioned distances on NPSDs.
Recall that the KNN algorithm is a non-parametric classifier \cite{weinberger2009distance, andoni2018approximate}, where sample $s$ is assigned to the class most common across its $k$ nearest neighbours. 
This strategy heavily relies on the distance (or divergence) used, where the predicted class follows directly from the distances among  samples.
Therefore, we implement the KNN concept to NPSDs using the above metrics thus yielding the classifiers $\kNN_{W_2}$, $\kNN_{\KL}$, $\kNN_{\IS}$ and $\kNN_{W_2}$. 
We cross-validate (on 10 folds) $k \in \{2, \cdots, 10 \}$ for each task and model. Results reported were obtained with the best performing $k$ for each classifier.

\subsection{Binary classification of two real-world datasets}
\label{sec:binary_real}

We now perform binary classification using the logistic and KNN models with the four distances (8 models in total) applied to two real-world datasets: Urbansound8k and GunPoint.

\textbf{Urban audio}. Urbansound8k \cite{salamon2014dataset} comprises 4-second recordings of the classes: \emph{air conditioner (AC), car horn, children playing, dog bark, drilling, engine idling (EI), gunshot, jackhammer (J), siren}, and \emph{street music} (see Table \ref{table:urban} in Appendix \ref{appendix:dataset} for number of samples). We considered foreground-target recordings at 44100 Hz, that is, a dataset of 3487 one-dimensional labelled NPSDs\footnote{PSDs were computed using the periodogram.}. We focused on binary classification related to discriminating class \textit{gun shot} from the remaining classes, thus having 9 different tasks.

We trained $\LL_W$, $\LL_{\L_2}$, $\LL_{\IS}$ and $\LL_{\KL}$ via maximum likelihood\footnote{The negative-log-likelihood was minimised using the L-BFGS-B algorithm.} and cross-validated using the 10-fold as recommended by the dataset. Table \ref{tab:urbansound_lr} shows the mean accuracy of all 8 models and tasks with their 10-fold standard error, with the best and second-to-best performing methods shown in boldface black and blue respectively. Out of the nine tasks, WF is the top performer metric (through KNN) four times, and second-best performer metric (through logistic reg.) three times. Furthermore, in the three tasks where WF performed second best, its performance is well within the error bars of the top performer. Another conclusion that can be drawn is that WF seems to performs better on KNN, while IS performed better on the logistic reg. 
\CR{We completed our analysis with a logistic regressor that combines the features of $\L_2, \KL$ and $W_2$ by considering a linear combination of all three distances, referred to as $\mathcal{L}_{ALL}$. This aggregated model outperformed other logistic regressors in 5 out of 9 tasks, providing an accuracy improvement of up to 0.04. This supports the idea that different features conveyed by the considered metrics can be successfully combined.}

\textbf{Body-motion signals}. The Gunpoint dataset \cite{ratanamahatana2004everything} contains hand trajectories of length $150$ from \emph{young} and \emph{old} actors \emph{drawing} a gun or just \emph{pointing} at a target with his finger. Therefore, the dataset naturally comprises two pairs of non-overlapping classes: \emph{young} VS \emph{old}, and \emph{drawing} VS \emph{pointing}. See Table \ref{table:motion} in Appendix \ref{appendix:dataset} for number of samples.

For each pair of classes we performed binary classification using $10$ random splits of the dataset, using $80\%$ as training data and $20\%$ as test data. Figure \ref{fig:gunpoint_perf} shows the performance in the form of box plots. Notice, for both tasks, that the top performing methods are $\kNN_{\KL}$ and $\kNN_{W_2}$ with tight error bars. For the particular case of the logistic reg.~methods, $\LL_{W_2}$ appears to be the best (\emph{draw} VS \emph{point}) and second best (\emph{young} VS \emph{old}) performing option.
Regarding computation time, we emphasise that all distances considered belong to the same complexity class (up to a constant factor). 
This constant factor is relatively small (on the draw VS point task) in the logistic regression framework: $\LL_{W_2}$ took 0.23 seconds to complete training in average while competing methods took 0.09 seconds.
The impact of the constant factor is more visible in the $\kNN$ framework where WF took in average 3.7 seconds while other methods around 0.1 seconds, for the average training time.

\begin{figure}[!ht]
	\includegraphics[width=\linewidth]{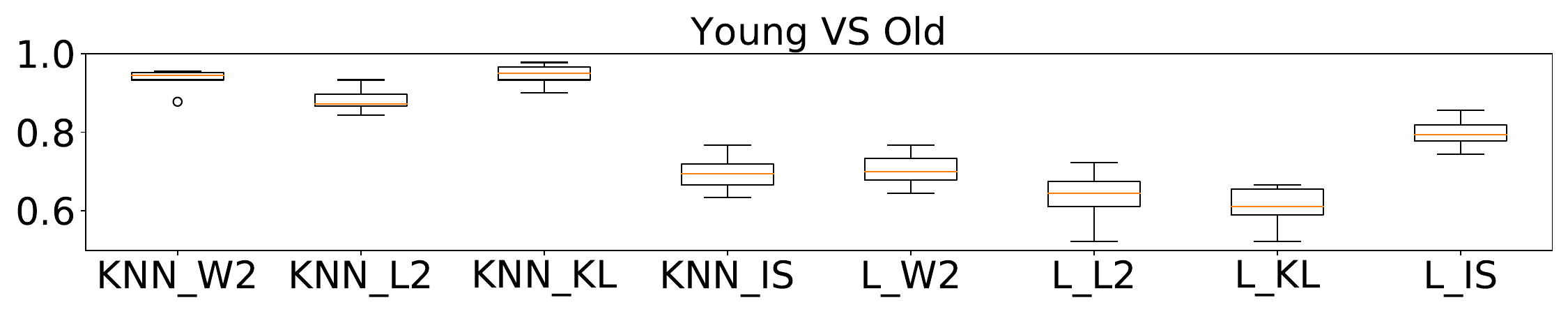}
	\includegraphics[width=\linewidth]{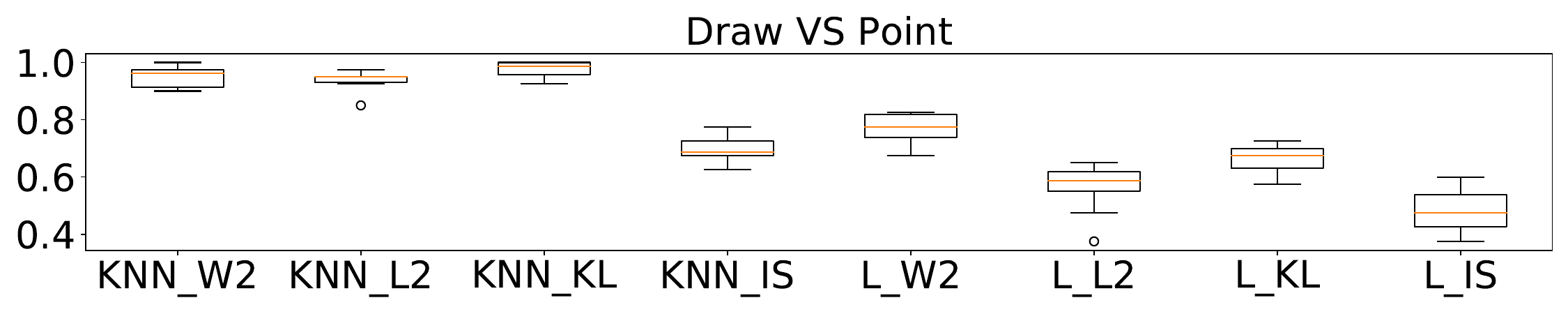}
	
	\caption{Classification of Gunpoint pairs: young vs old (top) and draw vs point (bottom). Accuracy and standard error using 10 different train-test ($80-20$) splits.}
	\label{fig:gunpoint_perf}
\end{figure}

%!TEX root = ../CRT_TSP2019_arxiv.tex

\section{Concluding remarks}
We have proposed the Wasserstein-Fourier based distance (WF) to study stationary time series. The WF distance has been presented in (i) theoretical terms, (ii) connection with consistency of the empirical autocorrelation function, (iii) data interpolation and augmentation, (iv) dimensionality reduction, and (v) time series classification. These results establish WF as a sound and intuitive metric for modern time series analysis, and thus working towards bridging the gap between the computational optimal transport and signal processing communities. The proposed WF distance is, to the best of our understanding, unique within the literature in the following aspects: 
\begin{itemize}
 	\item It provides a \emph{horizontal} assessment of spectral content, thus allowing to explicitly quantify the displacement of the power spectrum (frequency shift or warping). For instance, the WF distance between cosines of frequencies $\omega_1$ and $\omega_2$ is simply $|\omega_1-\omega_2|$.
 	\item It is able to compare PSDs of different support, critically, it can compare continuous PSDs against discrete ones. For instance, it can measure the distance between a sinc function and a sum of cosines.
 	\item It is a proper distance (unlike KL or IS) fulfilling the three axioms of a distance, and it is also meaningful from a spectral analysis perspective (unlike the Euclidean distance). This allows us translate existing distance-based methods directly to the space of time series such as dimensionality reduction, regression, classification, clustering, synthesis, or data augmentation.
 \end{itemize} 
In terms of computational overhead, though all distances considered are in the same order of complexity, we acknowledge that WF exhibited computation times larger than its competitors. This, however, is not detrimental for the implementation of the proposed distance, since i) WF has unique properties such as those mentioned above and therefore greater computational complexity is expected, and ii) operation on real-world examples with hundreds of time series still remained in the orders of second, both for training and prediction. Lastly, the computation of the WF distance can still be accelerated based on current developments for fast computations of the Wasserstein distance (e.g. \cite{sinkhorn_lightspeed}).

 Future research includes exploiting the property of geodesic path between GPs, briefly illustrated in this paper, build new training procedures for Gaussian processes, as well as using optimal transport for unbalanced distributions \cite{chizat2018interpolating}, and therefore directly on the space of PSDs without the need for normalisation.

%!TEX root = ../CRT_TSP2019_arxiv.tex

\section{Appendix} 
\label{sec:appendix}

\subsection{Non-convexity example of the WF distance}
\label{appendix:nonconv}

The last equality of eq.~\eqref{eq:nonconv} in Sec.~\ref{sec:properties}, which serves as a counterexample for the convexity of the proposed WF distance, follows from   
\begin{align*}
& \hat{x}(\xi) = \delta(\xi-\frac{\omega}{2\pi})+\delta(\xi+\frac{\omega}{2\pi})\\
& \widehat{y^{+}}(\xi) = \hat{x}(\xi) +\frac{1}{r}(\delta(\xi-\frac{\alpha}{2\pi}) +\delta(\xi+\frac{\alpha}{2\pi}))
\end{align*}
leading to 
\begin{align*}
& S_{x}(\xi)= \delta(\xi-\frac{\omega}{2\pi}) +\delta(\xi+\frac{\omega}{2\pi})\\
& S_{y^{+}}(\xi) = S_{x}(\xi) + \frac{1}{r^2}(\delta(\xi-\frac{\alpha}{2\pi}) +\delta(\xi+\frac{\alpha}{2\pi}))
\end{align*}
Normalising the PSDs, we get
\begin{align*}
& s_{x}(\xi)= \frac{1}{2}(\delta(\xi-\frac{\omega}{2\pi}) +\delta(\xi+\frac{\omega}{2\pi}))\\
& s_{y^{+}}(\xi)=\left(\frac{r^2}{r^2+1}\right)s_{x}(\xi)+ \left(\frac{1}{2(r^2+1)}\right)(\delta(\xi-\frac{\alpha}{2\pi})\\ & \hspace{4cm}+\delta(\xi+\frac{\alpha}{2\pi})).
\end{align*}
Therefore, moving the mass of $s_{y^{+}}$ onto $s_{x}$ boils down to moving the mass located in the Dirac $-\frac{\alpha}{2\pi}$ and $\frac{\alpha}{2\pi}$ onto respectively the Dirac $-\frac{\omega}{2\pi}$ and $\frac{\omega}{2\pi}$, which is twice the same cost. Therefore, we obtain
\begin{align}
 	\WF{y^{+}}{x}&= \left(\frac{2}{2(r^2+1)}\vert \frac{\omega}{2\pi}-\frac{\alpha}{2\pi}\vert^2\right)^{1/2}\\
 	&=\vert \omega-\alpha\vert / (2\pi\sqrt{r^2+1}).\nonumber
 \end{align}
 The derivation of $\WF{y^{-}}{x}$ works similarly.

%%%%%%%%%%%%%
\subsection{Proof of Proposition \ref{prop:conv}}
\label{appendix:proof}

\begin{proof}
By the normalised version of the Bochner theorem in eq.~\eqref{eq:Wiener2}, and since the Fourier transform and its inverse are continuous, when $n$ tends to infinity one has that $\brn(h)\to r(h)$ is equivalent to $\bsn(\xi)\to s(\xi)$, the NPSDs associated to $\byn$ and $y$. We next prove the right and left implications respectively. Note that the convergence in the Wasserstein metric is equivalent to the weak convergence and the convergence of the moments of order $p$ at the same time (see \textit{e.g.} \cite{villani2003topics}).

\textbf{[$\Rightarrow$]} Let us assume that $\bsn(\xi)\to s(\xi)$. By Scheff\'e's lemma, $\bsn$ weakly converges to $s$. Moreover, since $\bsn$ is supported on a compact (as $y$ is band-limited), the dominated convergence theorem guarantees that $\lim_{n\to\infty} \int \vert \xi\vert^2\td\bsn(\xi)=\int \vert \xi\vert^2 \td s(\xi)$. This implies that $\lim_{n\to\infty}W_2(\bsn,s)= 0$ (see e.g. Theorem 7.12 in \cite{villani2003topics}) and thus $\lim_{n\to\infty}\WF{\byn}{y} = 0$.

\textbf{[$\Leftarrow$]} Conversely, if $W_2(\hat{s}_n,s)\to 0$ when $n\to\infty$, then $\bsn$ weakly converges to $s$, and by the L\'evy's continuity theorem, we have that $\lim_{n\to\infty}\int e^{jh\xi}\td\bsn(\xi)= \int e^{jh\xi}\td s(\xi)$, which directly implies that $\lim_{n\to\infty} \brn(h) = r(h)$ for all $h$. 
\end{proof}

%%%%%%%%
\subsection{Principal geodesic analysis: a brief review}
\label{appendix:PCA}

We rely on the method proposed by \cite{fletcher2004principal,petersen2016functional}, which is next adapted to our one-dimensional setting  in a simplified manner, after defining the Wasserstein mean of distributions.
Let $\{s_n\}_{n=1}^N$ be a set of $N$ distributions %with Wasserstein barycenter $\bar{s}$:
in $\PP_2(\R^d)$. A Wasserstein barycenter $\bar{s}$---or mean---of this family, is defined as the Fr\'echet mean with respect to the Wasserstein distance\footnote{In one dimension, the barycenter is closed form: $F^{-}_{\bar{s}} =\frac{1}{n}\sum F^{-}_{s_i}$.}, as introduced by \cite{agueh2011barycenters}. In other words, the Wasserstein barycenter  is a solution of the following minimisation problem:
\begin{equation}
\label{def:barycenter}
\bar{s} \in \argmin_{s\in\PP_2(\R^d)} \ \frac{1}{N}\sum_{n = 1}^{N} W_2^2(s_n,s).
\end{equation}
The PGA procedure is then as follows:
\begin{enumerate}
 	\item For each distribution $s_n$, compute the projection $\ell_n$ onto the tangent space at the barycenter $\bar{s}$ using the logarithmic-map according to 
 	\begin{equation}
 		\ell_n = \log_{\bar{s}}(s_n) := F^{-}_{s_n}\circ F_{\bar{s}}-\operatorname{id},
 	\end{equation}
 	where recall that $F_{\bar{s}}$ denotes the cumulative function of $\bar{s}$ and $F^{-}_{s_n}$ the inverse cumulative function of $s_n$.
 	\item Perform FPCA on the log-map projections $\{\ell_n\}_{n=1}^N$ in the Hilbert space $\L_2$ weighted by the barycenter $\bar{s}$ (namely  $\L_{2,\bar{s}}$), using $K$ components. Therefore, we can denote  the projection of $\ell_n$ onto the $k\textsuperscript{th}$ component as
 	 \begin{equation}
 		p_{k,n}:=\hat{\ell} +t_{k,n}v_k,
 	\end{equation}
 	where $v_k$ is the $k\textsuperscript{th}$ eigenvector of the log-map's covariance matrix, $t_{k,n}$ is the location of the projection of $\ell_n$ onto $v_k$, and $\hat{\ell}$  is the Euclidean mean of $\{\ell_n\}_{n=1}^N$.
 	\item Map each $\{p_{k,n}\}_{k=1,n=1}^{K,N}$ back to the original distribution space using the exponential map 
 	\begin{equation}
 		g_{k,n} = \exp_{\bar{s}}(p_{k,n}) := (\operatorname{id}+p_{k,n})\#\bar{s},
 	\end{equation}
 	where the pusforward operator $\#$ reads for a function $T$ as $( T  \# \bar{s}  )(A) =  \bar{s}\{x\in\R^d | T(x)\in A\}$, for any $A\subset\R^d$.
 \end{enumerate} 
 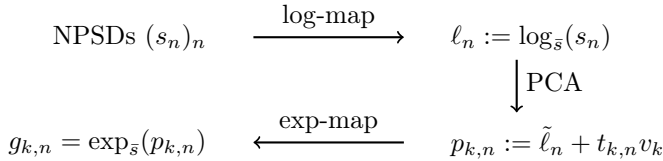
\begin{figure}[ht!]
\centering
\begin{tikzpicture}
\draw [black] (0,0) node[left]{NPSDs $(s_n)_n$};
\draw [->] [black, thick] (0.5,0) to (2.5,0);
\draw [black] (1.5,0) node[above]{$\log$-map};
\draw [black] (3,0) node[right]{$\ell_n:=\log_{\bar{s}}(s_n)$};
\draw [->] [black, thick] (4,-0.3) to (4,-1);
\draw [black] (4,-0.6) node[right]{PCA};
\draw [black] (3,-1.4) node[right]{$p_{k,n}:= \tilde{\ell}_n + t_{k,n}v_k$};
\draw [->] [black, thick] (2.5,-1.4) to (0.5,-1.4);
\draw [black] (1.5,-1.4) node[above]{$\exp$-map};
\draw [black] (0,-1.4) node[left]{$g_{k,n}=\exp_{\bar{s}}(p_{k,n})$};
\end{tikzpicture}
\caption{PGA of NPSDs in the frequency domain where $\tilde{\ell}_n$ is the Euclidean mean of $(\ell_n)_n$, $(v_k)_k$ are the eigenvectors of the covariance matrix of the $\log$-maps $(\ell_n)_n$ and $(t_{k,n})$ is the inner product of the $n\textsuperscript{th}\ \log$-map with the $k\textsuperscript{th}$ eigenvectors.}
\label{fig:PGA}
\end{figure}
As a result, observe that for every pair $(k,n)$, $g_{k,n}$ is a valid distribution that corresponds to the projection of the $n\textsuperscript{th}$ original distribution $s_n$, onto the $k\textsuperscript{th}$ geodesic component $g_{k}$ defined by $g_k(t) : = \exp_{\bar{s}}(\hat{\ell} +t v_k)$, which is a distribution for each $t\in\R$. This is illustrated in Fig. \ref{fig:PGA}.

%projecting the NPSD $s$ into the tangent space at the barycenter $\bar{s}$ (see eq. \eqref{def:barycenter} [FUTURE?]) via the logarithmic map ($\log$-map [!]) and then perform standard PCA on the $\log$-map [!] dataset (i.e., via the eigenvectors of the empirical covariance matrix). More explicitly, the log-map of $s$ at $\bar{s}$ is given by $\log_{\bar{s}}(s) := F^{-}_s\circ F_{\bar{s}}-\operatorname{id}$, where.... Recall that during classical PCA, each data vector is projected along the eigenvectors \CR{[principal components?]}, at some $t\in\R$, which corresponds to the inner product. Once the PCA has been done, it remains to map back the projection $p_{k,n}$ of each $\log$-data $\log_{\bar{s}}(s_n)$ along the $k\textsuperscript{th}$-eigenvector $v_k$ into the NPSD space through the exponential map $\exp_{\bar{s}}(p_{k,n}) := (\operatorname{id}+p_{k,n})\#\bar{s}$. For each pair $(k,n)$, this actually is a NPSD $g_{k,n}$ that corresponds to the projection of the $n\textsuperscript{th}$ data into the $k\textsuperscript{th}$ geodesic component $g_{k}$ in the frequency space. This is recapitulated in Fig. \ref{fig:PGA}.

Although computationally fast, the above procedure might not provide exact PGA, since a geodesic in the Wasserstein space is not exactly the image under the exponential map of straight lines in $\L_{2,\bar{s}}$ (see \cite{bigot2017geodesic} for more details). Alternative methods are those proposed by \cite{cazelles2018geodesic} based on \cite{bigot2017geodesic}, and \cite{masarotto2019procrustes} which performs PCA on Gaussian processes by applying standard PCA on the space of covariance kernel matrices (linked to the PSDs through Bochner's theorem).

\subsection{Algorithm for PGA in Wasserstein-Fourier}
\label{appendix:PCA_algo}

\CR{The following Algorithm \ref{algo:PGA} is developed in \url{https://github.com/GAMES-UChile/Wasserstein-Fourier}.}
\begin{algorithm}
\CR{{\bf Input:} $N$ time series $(x_n)_{i=n,\ldots,N}$\\
{\bf 1:} Compute the corresponding NPSDs $(s_n)_{n=1,\ldots,N}$\\
{\bf 2:} Compute the quantile functions $(F^{-}_{s_n})_{n=1,\ldots,N}$\\
{\bf 3:} Compute the quantile of the Wasserstein barycenter $F^{-}_{\bar{s}} =\frac{1}{N}\sum_{n=1}^N F^{-}_{s_n}$ and then its cumulative function $F_{\bar{s}}$\\
{\bf 4:} For each $n=1,\ldots,N$, compute the projection $\ell_n = F^{-}_{s_n}\circ F_{\bar{s}}-\operatorname{id}$, and their Euclidean mean $\bar{\ell}=\frac{1}{N}\sum_{n=1}^N \ell_n$\\
{\bf 5:} Compute the empirical covariance operator $C\ell = \frac{1}{N}\sum_{n=1}^N\langle \ell_n-\bar{\ell},\ell\rangle_{\bar{s}}\ (\ell_n-\bar{\ell})$\\
{\bf 7:} Compute $(v_k)_{1\leq k \leq K}$, the eigenvectors associated to the $K$ largest eigenvalues of $C$, and the location of the projection of $\ell_n$ onto $v_k$, namely $t_{k,n} = <\ell_n-\hat{\ell},v_k>_{\bar{s}}$\\
{\bf 8:} Construct the projections $p_{k,n}:=\hat{\ell} +t_{k,n}v_k$\\
{\bf 9:} Compute the geodesic principal component back into the Wasserstein space $g_{k,n} = (\operatorname{id}+p_{k,n})\#\bar{s}$ 	

 \caption{PGA in Wasserstein-Fourier}}
\label{algo:PGA}
\end{algorithm}

\subsection{Number of samples for datasets of Section \ref{sec:log_reg}}
\label{appendix:dataset}

\begin{table}[ht!]
\centering
\begin{tabular}{|l|l|}
 class name & $\#$ samples   \\ \toprule
 air conditioner & 442    \\
 car horn & 88      \\
 children playing & 324 \\
 dog barking & 367 \\
 drilling & 646 \\
 engine idling & 455 \\
 gun shot & 132 \\
 jackhammer & 408 \\
 siren & 189 \\
 street music & 436
\end{tabular}
\vspace{0.2cm}
\caption{Urban sound dataset}
\label{table:urban}
\end{table}
\begin{table}[ht!]
\centering
\begin{tabular}{|l|l|}
 class name & $\#$ samples   \\ \toprule
 draw & 100    \\
 point & 100      \\
 old & 215 \\
 young & 236 
\end{tabular}
\vspace{0.2cm}
\caption{Body-motion dataset}
\label{table:motion}
\end{table}

\end{document}